\documentclass{article}
\pdfoutput=1 
\usepackage{natbib}

\usepackage{arxiv}

\usepackage[utf8]{inputenc} 
\usepackage[T1]{fontenc}    
\usepackage{hyperref}       
\usepackage{url}            
\usepackage{booktabs}       
\usepackage{amsfonts}       
\usepackage{nicefrac}       
\usepackage{lipsum}

\usepackage{microtype}
\usepackage{graphicx}

\usepackage{amsmath}
\usepackage{amssymb}
\usepackage{mathtools}
\usepackage{amsthm}

\usepackage[capitalize,noabbrev]{cleveref}

\theoremstyle{plain}
\newtheorem{theorem}{Theorem}[section]
\newtheorem{proposition}[theorem]{Proposition}

\theoremstyle{definition}

\theoremstyle{remark}

\usepackage{amsmath,amsfonts,bm}

\def\eqref#1{equation~\ref{#1}}

\def\1{\bm{1}}

\def\vu{{\bm{u}}}

\DeclareMathAlphabet{\mathsfit}{\encodingdefault}{\sfdefault}{m}{sl}
\SetMathAlphabet{\mathsfit}{bold}{\encodingdefault}{\sfdefault}{bx}{n}

\newcommand{\E}{\mathbb{E}}

\newcommand{\R}{\mathbb{R}}

\usepackage{amsmath}
\usepackage{amsthm}

\theoremstyle{remark}
\usepackage{thm-restate}

\usepackage{bbm}

\usepackage{titletoc}
\usepackage[page,header]{appendix}

\usepackage{url}
\usepackage{graphicx}
\usepackage{tabularx}
\usepackage{multirow}
\usepackage{makecell}
\usepackage{microtype}
\usepackage{colortbl}
\usepackage{booktabs}
\usepackage{xspace}
\usepackage[font=small]{caption}
\usepackage{subcaption}
\usepackage{sidecap}
\usepackage{wrapfig}
\usepackage{algorithm}
\usepackage{listings}

\def\tco{{\it train-clean-100}}

\def\devclean{{\it dev-clean}}
\def\devother{{\it dev-other}}
\def\testclean{{\it test-clean}}
\def\testother{{\it test-other}}
\newcommand{\librispeech}{{LibriSpeech}}

\newcommand{\sigmareparam}{$\sigma$Reparam\xspace}

\newcommand{\hparams}{hyperparameters\xspace}
\newcommand{\hparam}{hyperparameter\xspace}
\newcommand{\Hparams}{Hyperparameters\xspace}

\usepackage{xcolor}
\definecolor{ApplePrimaryChartRed}{RGB}{227,94,105}

\global\long\def\R{\mathbb{R}}%

\global\long\def\b\sigma{\boldsymbol{\sigma}}%

\usepackage[textsize=tiny]{todonotes}

\title{Stabilizing Transformer Training by Preventing Attention Entropy Collapse}

\author{
  Shuangfei Zhai\thanks{Equal contribution}~~, Tatiana Likhomanenko$^*$, Etai Littwin$^*$, Dan Busbridge$^*$, Jason Ramapuram$^*$, \\ {\bf Yizhe Zhang, Jiatao Gu, Josh Susskind} \\
  Apple \\
  \texttt{\{szhai,antares,elittwin,dbusbridge,jramapuram,yizzhang,jgu32,jsusskind\}@apple.com} \\
\\
}

\begin{document}
\maketitle

\begin{abstract}
Training stability is of great importance to Transformers. 
In this work, we investigate the training dynamics of Transformers by examining the evolution of the attention layers. 
In particular, we track the attention entropy for each attention head during the course of training, which is a proxy for model sharpness. 
We identify a common pattern across different architectures and tasks, where low attention entropy is accompanied by high training instability, which can take the form of oscillating loss or divergence. 
We denote the pathologically low attention entropy, corresponding to highly concentrated attention scores, as \emph{entropy collapse}. 
As a remedy, we propose \sigmareparam, a simple and efficient solution where we reparametrize all linear layers with spectral normalization and an additional learned scalar. 
We demonstrate that \sigmareparam successfully prevents entropy collapse in the attention layers, promoting more stable training. Additionally, we prove a tight lower bound of the attention entropy, which decreases exponentially fast with the spectral norm of the attention logits, providing additional motivation for our approach.
We conduct experiments with \sigmareparam on image classification, image self-supervised learning, machine translation, speech recognition, and language modeling tasks.
We show that \sigmareparam provides stability and robustness with respect to the choice of hyperparameters, going so far as enabling training 
(a) a Vision Transformer {to competitive performance} without warmup, weight decay, layer normalization or adaptive optimizers; 
(b) deep architectures in machine translation and 
(c) speech recognition to competitive performance without warmup and adaptive optimizers.
Code is available at {\small \url{https://github.com/apple/ml-sigma-reparam}}.

\end{abstract}

\keywords{Transformers, self-attention, optimization, stability, spectral normalization, self-supervised learning, vision, speech, language, contrastive learning}

\startcontents[sections]
\printcontents[sections]{l}{1}

\begin{figure}[ht!]
\centering
\begin{minipage}{.48\textwidth}
  \input{newfig1/old_fig1.tex}
\end{minipage}%
\hspace{0.3cm}
\begin{minipage}{.48\textwidth}
  \input{newfig1/fig1-intervention.tex}
\end{minipage}
\end{figure}

\section{Introduction}
Transformers~\citep{vaswani2017attention} are state-of-the-art models in many application domains.
Despite their empirical success and wide adoption, great care often needs to be taken in order to achieve good training stability and convergence. In the original paper~\citep{vaswani2017attention}, residual connections and Layer Normalizations (LNs)~\citep{ba2016layer} are extensively used for each attention and MLP block (specifically, in the post-LN fashion). There has since been various works attempting to promote better training stability and robustness. For example, the pre-LN~\citep{radford2019language} scheme has gained wide popularity, where one moves the placement of LNs to the beginning of each residual block. Others have argued that it is important to properly condition the residual connections. \citet{bachlechner2021rezero} proposes to initialize the residual connections to zero to promote better signal propagation. \citet{zhang2018fixup, huang2020improving} remove LNs with carefully designed initializations. 

In this work, we study the training instability of Transformers through the lens of training dynamics. We start by monitoring the entropy of the attention maps averaged over all query positions, heads and examples.
We have found that the attention entropy is tightly correlated with the model's stability and convergence. 
In particular, small attention entropy is often accompanied with slow convergence, fluctuations in training loss and, in the worst case, divergence. As a motivator, we plot the attention entropy curves of a highly optimized Vision Transformer (ViT)~\citep{dosovitskiy2020image,touvron2021training} in Figure \ref{fig:vit_entropy}. We observe an initial loss oscillation happening at the same time with sharp dips of the attention entropy curves. When doubling the default learning rate, all attention entropy collapses to near zero and training diverges. 
In addition, we show in \Cref{fig:ssl-stats-40,fig:mt-stability-short} two sets of experiments of baseline Transformers models with training instability occurring at the same time of entropy collapse. And more generally, similar observations can be made in a wide range of model/task settings if hyperparameters such as learning rate, warmup, initialization are not carefully tuned. 

To further demonstrate this connection, we modify the Transformer to have a global temperature by dividing the pre-softmax (logits) matrix of each attention mechanism by a scalar quantity whose default value is 1.
Modifying the temperature gives direct control over the attention entropy, enabling the investigation of a causal connection between entropy collapse and training instability (see \Cref{fig:interventions} and \Cref{fig:causal-temp-start,fig:causal-temp-value} in Appendix~\ref{sec:app-causal}).
Here we train a ViT-B/16 on ImageNet1k. At an \emph{intervention epoch} we modify the temperature from its default value to 0.1.
We see that when performing this intervention during warmup, attention entropy drops to near zero and training becomes unstable. 
A late intervention also causes a drop in entropy and accuracy curves, however, the model is able to recover to a higher attention entropy regime, although yielding a lower accuracy than non-intervened training.

To further understand these phenomena, we computed the \emph{sharpness} -- the largest singular value of the Hessian (the second order derivative of the loss with respect to the model parameters), as its magnitude has implications for training stability
\cite{
DBLP:conf/icml/GhorbaniKX19,
DBLP:conf/bigdataconf/YaoGKM20,
cohen2021gradient,
DBLP:journals/corr/abs-2207-14484,
DBLP:journals/corr/abs-2110-04369}.
When sharpness exceeds an algorithm-dependent stability threshold, training iterations diverge
\cite{cohen2021gradient,DBLP:journals/corr/abs-2207-14484}.
We see that interventions inducing the largest drop in attention entropy result in the sharpness exceeding the stability threshold, whereas the later interventions do not cause the threshold to be crossed, explaining how they can recover. 
For details on the empirical setup and additional results see \Cref{sec:app-causal}.

The empirical correlation of entropy collapse and training instability leads to the following questions: 
1) How do we prevent entropy collapse? 
2) Can we improve training stability by doing so? 
We answer these by showing that {entropy collapse can be effectively prevented by controlling the spectral norms of the query and key projections. 
In particular, we prove a tight lower bound on the attention entropy, which decreases exponentially fast with the growth of the spectral norm of the attention matrix logits. This bound suggests that entropy collapse can occur swiftly when letting the spectral norm of the weights increase uncontrollably}. We then provide a simple fix, \sigmareparam, which reparameterizes all weight matrices by sequentially applying Spectral Normalization~\citep{miyato2018spectral} and a learned multiplicative scalar. Intuitively, $\sigma$Reparam decouples the update of the spectral norms of weights from their dimensionality, which allows them to update smoothly and in a controlled way. 
Also note that \sigmareparam does not change the model space, which allows one to learn an equally expressive model. 

We evaluate five tasks: image classification, self-supervised learning (SSL), machine translation, automatic speech recognition (Appendix~\ref{sec:asr}), and language modeling (Appendix~\ref{sec:language-modeling}). We highlight the empirical results as follows:

    1. We show that entropy collapse is commonly observed in the baseline models of various benchmarks.
    
    2. Image classification: \sigmareparam enables a drastically simplified ViT training recipe by removing pre-LN, learning rate warmup, weight decay and not requiring adaptive optimizers. This recipe leads to equivalent (or slightly better) model performance against baseline training strategies, all the while reducing training duration by 16\% .
    
    3. Self-supervised learning: \sigmareparam helps to drastically improve the stability and robustness of the SimCLR training, improving upon existing baselines.
    
    4. Machine translation: \sigmareparam allows us to stabilize very deep post-LN architectures up to 100L-100L encoder-decoder layers.
    
    5.  Speech recognition: \sigmareparam allows us to improve training stability and simplify 
    the training recipe for post-LN Transformer by removing learning rate warmup and adaptive optimization.
    
    6. Language modeling: \sigmareparam is compatible with causal Transformer architectures, and achieves results competitive with state-of-the-art without using post-LN.

\section{Related Works}
Transformers have relied heavily on LNs to achieve training stability. Besides the popular post-LN and pre-LN configurations, other variants have been proposed~\citep{wang2022deepnet,shleifer2021normformer,liu2020admin}. On the one hand, we show empirically that entropy collapse (and its accompanied training instability) happens even equipped with extensive use of normalization layers. On the other hand, \sigmareparam does not rely on specific normalization layers and can even work in the absence of it, while effectively smoothing the attention entropy curves. 

There have also been numerous attempts to design better Transformer initialization schemes, including~\citet{zhang2018fixup,huang2020improving,yang2022tensor,bachlechner2021rezero}. While proper initializations are indeed crucial to stable and fast training, we argue that the training dynamics (affected by the optimizer and training hyperparameters) is equally important. \sigmareparam in this sense is an orthogonal approach that specifically targets the entropy collapse problem, which makes it compatible with standard initialization methods and provides robust performance. 

\sigmareparam is a special case of weight reparameterization, which has found wide adoption in deep learning. WeightNorm (WN)~\citep{salimans2016weight} is a well known example of such methods, but its effectiveness in Transformers is limited. In ConvNets, simple additive weight reparameterization~\citep{ding2021repvgg} has been demonstrated useful in speeding up training convergence. To the best of our knowledge, \sigmareparam is the first simple reparameterization technique that provides competitive performance with well optimized baseline models. Normalizing weights by its spectral norm is also inspired by SpectralNorm~\citep{miyato2018spectral}, with the key difference that SpectralNorm explicitly constrains the model's capacity, which brings significant performance loss.

Another related line of work is the rank collapse of Transformer training, first identified by~\citep{dong2021attention}. Rank collapse refers to the degenerate state of attention where its output converges to a rank 1 matrix, where all tokens share the same representation. This analysis is further followed up by~\citep{noci2022signal} suggesting that rank collapse causes vanishing gradient of the attention query and keys. Entropy collapse, on the other hand, characterizes a different failure pattern, where the attention matrix remains high rank, and it tends to introduce high gradient norms rather than vanishing gradients (see Figure \ref{fig:ssl-stats-40}).

\section{Method}
\label{sec:method}
\subsection{Attention Entropy}
At the core of Transformers is dot-product attention. 
Let $X \in \R^{T \times d}$ denote an input sequence to an attention layer (we assume self-attention for simplicity of presentation), where $T,d$ are the number of tokens and the token dimension, respectively; and let $W_K,W_Q \in \R^{d \times n_a},W_V \in \R^{d \times n_v}$ denote the key, query and value matrices. 
A simple attention layer then computes $\text{Att}(X) = AXW_V$ where $A = \psi(a), a = XW_KW_Q^\top X^\top$ and $\psi$ is the row-wise softmax function. We define the attention entropy of a row $i$ of $A$ by $
\text{Ent}(A_i) = -\sum_{j=1}^T A_{i,j}\log(A_{i,j})$. Let $\text{Ent}(A) = \frac{1}{T}\sum_{i=1}^T\text{Ent}(A_i)$ denote the average attention entropy of $A$.
Our goal is to alleviate the entropy collapse problem and achieve a smooth evolution of the attention entropy through training.

We next investigate the properties of attention entropy. We show in the next theorem that $\text{Ent}(A)$ is directly connected to the spectral norm (the largest singular value) of $W_KW_Q^\top$. 

\vspace{0.2cm}
\begin{restatable}[Attention entropy lower bound]{thm}{spectral}\label{thm:bound}
Let 
$\sigma = \|W_KW_Q^\top\|_2$,
$\sigma_x = \|XX^\top\|_2$, 
$\b\sigma = \sigma\sigma_x$ and 
$\beta = \exp\left(-\b\sigma\sqrt{\tfrac{T}{T-1}}\right)$.
Then it holds that:
\begin{align}\label{eqn:bound}
\text{Ent}(A_i) &\geq \log\left(1 + (T-1)\,\beta\right)
+ \frac{\b\sigma\sqrt{T(T-1)}\,\beta}{1 + (T-1)\,\beta}.
\end{align}
Moreover, there exist inputs $X$ and weights $W_K,W_Q$ for which the lower bound in \cref{eqn:bound} is tight.

\end{restatable}
Therefore, for large $\sigma,T$, the minimum attainable entropy behaves like $\Omega(T\sigma e^{-\b\sigma})$, hence decreasing exponentially fast with $\b\sigma$. We note that the bound on the entropy in \cref{thm:bound} is tight in a sense that it is achievable for some inputs $X$. Proofs for Theorem~\ref{thm:bound} and the following Proposition are provided in Appendix \ref{sec:proof}.

{\textbf{Entropy collapse and training stability.}  Transformers are hard to train, requiring a careful tuning of a variety of hyperparameters. Notably, transformers can exhibit stages of training instability, with loss values oscillating uncontrollably, to the point of divergence. From a loss geometry perspective, we hypothesize that these regions of instability are caused when the weights enter a region of high curvature, a hypothesis supported by \cite{Chen2021WhenVT}, which showed that transformer models tend to converge to extremely sharp local minima. In this paper however, we step away from the loss geometry perspective and identify a novel empirical observation unique to the Transformer architecture. We observe that training instability and attention entropy collapse appear in tandem. Moreover, this observation is consistent across multiple settings and modalities (see \cref{fig:speech:pre-ln-collapse,fig:speech:lars,fig:mt-stability-short,fig:mt:postln,fig:ssl-stats-40}). Equipped with this observation, we might ask whether preventing attention collapse might in turn prevent training instability. We highlight that the affirmative answer provided in this paper could prove extremely practical, as attention entropy is easier to compute and potentially manipulate then directly tackling the loss geometry, which typically involves computing second derivatives, as in \cite{Foret2020SharpnessAwareMF}}. We next describe out method for preventing entropy collapse through a simple reparameterization scheme.

\subsection{\texorpdfstring{\sigmareparam}{Lg}}

\sigmareparam is a method to reparameterize the weights of a linear layer with:
\begin{equation} \label{eqn:sigmareparam}
    \widehat{W}=\frac{\gamma}{\sigma(W)}W,
\end{equation}
where $\sigma(W)\in\mathbb{R}$ is the spectral norm of $W$ and $\gamma\in\mathbb{R}$ is a learnable parameter, initialized to 1. In practice, $\sigma(W)$ can be computed via power iteration~\citep{mises1929praktische} as in SpectralNorm (SN)~\citep{miyato2018spectral}, see Algorithm~\ref{alg:sigmareparam} in Appendix~\ref{sec:app:impl} for a sketch implementation. Note that \sigmareparam brings little extra overhead as the power iteration mainly consists of two matrix vector products and is only performed on the parameters rather than activations. During inference, one can compute $\widehat{W}$ once and freeze it, which has the same cost of a regular linear layer.

\textbf{\sigmareparam decouples the update rate of spectral norm from the dimensionality of weights.}
{As is the case with other reparameterization techniques, \sigmareparam leaves the representational capacity of the network intact, however forces a different optimization dynamic. This property makes it distinct from SN, which explicitly constrains the model space. By absorbing the spectral norm $\sigma$ into a single parameter $\gamma$, \sigmareparam effectively forces the updates for $\gamma$ to be dimensionality independent.
This property is in contrast to the naive parameterization, where the spectral norm of weight matrices grows rapidly for large weight matrices when equipped with adaptive optimizers.}
To illustrate this, we adopt common assumptions in stochastic optimization, and model the stochastic gradients at some point in the optimization by $g = \mu + \epsilon \in \R^{w \times w}$, where $\mu$ is the mean and $\epsilon$ is a random variable with $\E [\epsilon] = \mathbf{0},\E [\epsilon^2] = n^2 \in \R^{w \times w}$. A typical Adam optimizer update attempts to approximate the following ideal update: $\Delta = \frac{\E [g]}{\sqrt{\E [g^2]}}$. The following proposition lower bounds the spectral norm of the ideal update $\sigma(\Delta)$:
\begin{restatable}[]{prop}{proposition}\label{prop}
It holds that:
\begin{align}\label{eqn:adaptive_bound}
\sigma(\Delta) \geq  \sqrt{w}\sqrt{1 - \frac{1}{w^2}\sum_{i,j=1}^w \frac{n_{i,j}^2}{\mu_{i,j}^2 + n_{i,j}^2}}.
\end{align}
\end{restatable}
The noise second moment $n^2$ is typically  in the order of $\mu^2$, hence \cref{eqn:adaptive_bound} indicates that the spectral norm of the ideal update should be large, growing linearly with $\sqrt{w}$. Moreover, for large batch sizes we would have 
$n^2 \ll1$, resulting in $\sigma(\Delta) \sim \sqrt{w}$
\footnote{This estimation would be exact for full batch optimization.}. 
While such a large spectral norm could be offset by a proper learning rate adjustment, this would be counterproductive since 1) a small learning rate typically induces inferior performance, and 2) architectures with layers of varying sizes, such as the case in Transformers, would require a per layer learning rate tuning. In contrast, \sigmareparam avoids this issue since the spectral norm of each layer is controlled by a single parameter $\gamma$, hence the size of its update does not scale with $w$ and is uniform across layers. 
This indicates \sigmareparam should provide the models of improved robustness with respect to learning rate and other related hyperparameters, by maintaining the spectral norm of the weights (and as a result the attention entropy) in a healthy regime.

\begin{table*}[t]
  \caption{Supervised image classification on ImageNet1k. The B/L/H refer to ViT-B/16, ViT-L/16 and ViT-H/14 variants respectively. The H and L variants have a known overfitting trend on this dataset \citep{he2022masked}. SN corresponds to the spectral normalization baseline \emph{without} the learnable scalar, while WN refers to the WeightNorm baseline. The WN configuration leads to immediate divergence without using pre-LN; we thus only report the result with WN + pre-LN.} \label{tab:vision_summary}
\centering
\small
\setlength{\tabcolsep}{3pt}
\begin{tabular}{l|cccc|cc}
  \toprule
  & DeiT (B) & \sigmareparam (B) & SN (B) & WN (B) & MAE (B/L/H) & \sigmareparam (B/L/H)\\
  \midrule
  Top-1 (\%) & 81.8  &\textbf{ 82.2} &69.81 &77.51 &82.1 / 81.5 / 80.90 & 81.88 / \textbf{82.41} / \textbf{81.09} \\
  Training Epochs & 300 & 300 & \textbf{250} & \textbf{250} & 300 / 200 / 200 & \textbf{250} / 300 / \textbf{170} \\    
  pre-LN     & Yes & \textbf{No} & \textbf{No} & Yes & Yes & \textbf{No} \\
  SGD & No & No & \textbf{Yes (LARS)} & No & No& \textbf{Yes (LARS)} \\
  Cosine Schedule & Yes & Yes & \textbf{No} & \textbf{No} & Yes & \textbf{No} \\
  LR Warmup & Yes & Yes & \textbf{No} & \textbf{No} & Yes & \textbf{No} \\
  Weight Decay & Yes & Yes & \textbf{No} & \textbf{No} & Yes & \textbf{No} \\
  \bottomrule
\end{tabular}
\end{table*}
\begin{figure*}[h!]
  \setlength{\tabcolsep}{3pt}
  \small
  \begin{minipage}{0.45\textwidth}
  \centering
  \includegraphics[width=\textwidth]{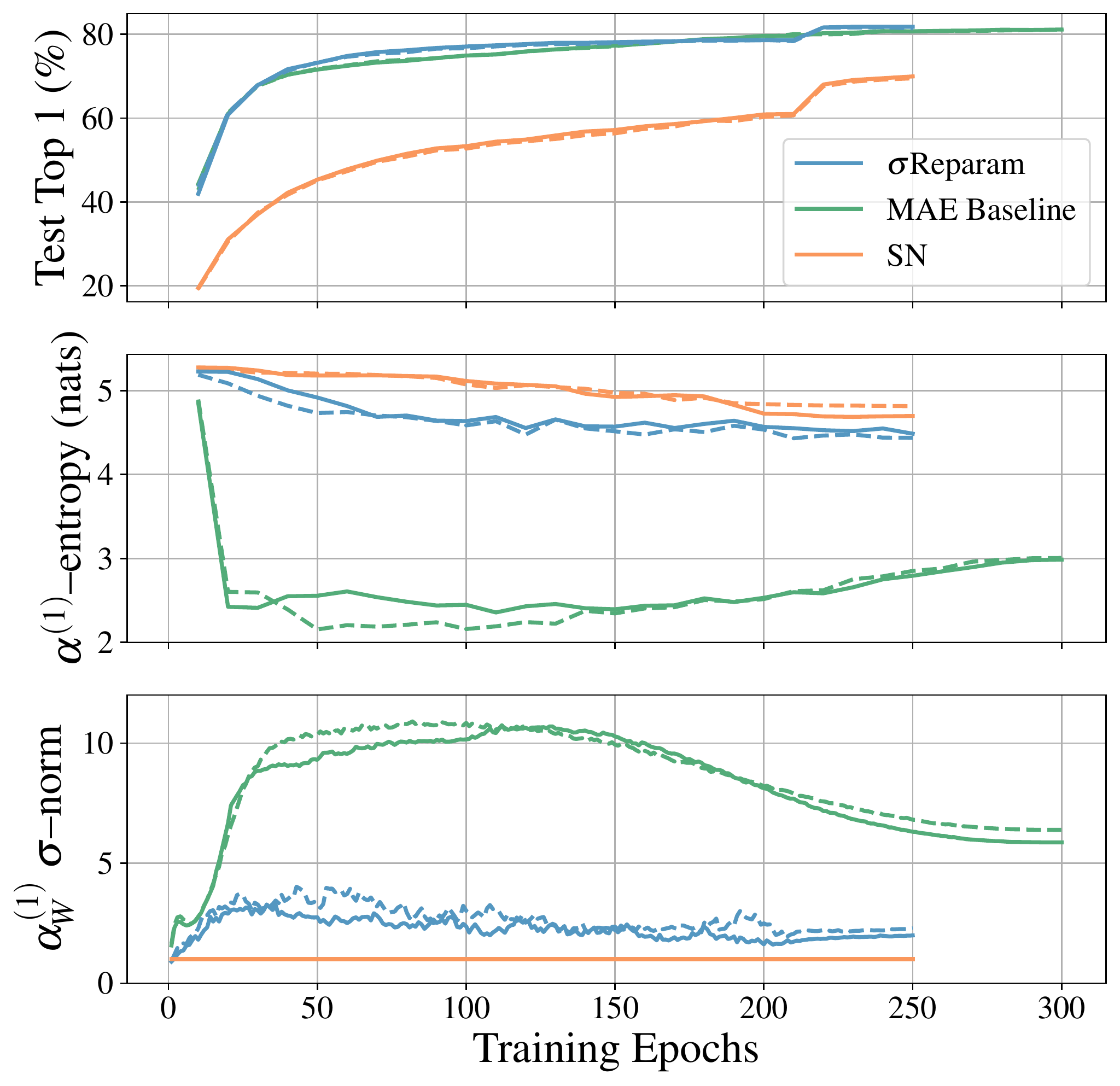}
  \end{minipage}%
  \hspace{0.2cm}
  \begin{minipage}{0.53\textwidth}
    \centering 
    \captionof{table}{Finetuned supervised image classification on ImageNet1k after pretraining on ImageNet21k (11M samples) or larger data. We compare \sigmareparam, trained for 90 epochs against DeiT3 \citep{DBLP:conf/eccv/TouvronCJ22} (trained for 90 [-90E] and 240 [-240E] epochs), an optimized finetuned CLIP \citep{DBLP:journals/corr/abs-2212-06138}, and a scaled supervised ViT-B trained on JFT-3B \citep{DBLP:conf/cvpr/Zhai0HB22}. All models compared use the ViT-B/16 architecture. \sigmareparam presents competitive results and sits in between the DeiT3-90E and DeiT3-240E runs, while not using pre-LN, LR warmup and only requiring a small weight-decay of $10^{-5}$.}\label{tab:imagenet21k}
    \resizebox{\columnwidth}{!}{
    \begin{tabular}{lccccc}
      \toprule
      & DeiT3-240E & DeiT3-90E & CLIP FT & ViT-B & \sigmareparam \\
      \midrule
      Test Top-1 (\%) & 86.7 & 85.2 & 86.6 & 86.6 & 85.84 \\
      EMA Top-1 (\%) & - & - & - & - & 85.87 \\        
      Dataset size & 11M & 11M & 400M & 3B & 11M \\
      Finetuning res & 384 & 224 & 384 & 384 & 384 \\
      pre-LN     & Yes & Yes & Yes & Yes & \textbf{No} \\
      Optimizer & LAMB & LAMB & AdamW & Adafactor & LAMB \\
      LR Schedule & Cos & Cos & Cos & r-sqrt & \textbf{step} \\
      LR Warmup & Yes & Yes & Yes & Yes & \textbf{No}  \\
      Weight Decay & Yes & Yes & Yes & Yes & Yes  \\
      \bottomrule
    \end{tabular}}      
  \end{minipage}
  \caption{\small ImageNet1k test performance, attention entropy, and largest singular value of attention weights of a supervised \sigmareparam ViT-B/16 alongside supervised MAE ViT-B/16 and spectral normalization (SN) baselines. Best (solid line) and worst (dashed line) trials of each method are presented. The MAE ViT-B/16 presents a more constrained attention entropy in contrast to the DeiT formulation from Figure \ref{fig:vit_entropy} due to the longer warmup, lower learning rate and stronger weight decay. While the SN baseline presents stable training, the model substantially under performs \sigmareparam. 
  } 
  \label{fig:joint_sup_vitb_and_21k}
\end{figure*}

\section{Experiments}
\subsection{Supervised Image Classification}\label{sec:supimg}

\textbf{Improved robustness.} We first start from a well tuned recipe with ViT-B on ImageNet1k \citep{deng2009imagenet,touvron2021training}, 
and vary its hyperparameters in the grid 
$\Big[
\texttt{baseLR} \in \{5\times 10^{-4}, 10^{-3}\}$, 
$\texttt{batchSize} \in \{1024, \,2048\}$, 
$\texttt{warmupEpochs} \in \{0, 5\}\Big]
$. 
{\bf 7/8 configurations} lead to divergence except for the default $\Big[5\times 10^{-4},$ $2048,$ $5\Big]$ hyperparameter. We next apply \sigmareparam to all the linear layers (including the initial patch embedding), and remove all the pre-LNs instances. All configurations in the same grid search converge with an average top-1 accuracy of 81.4\% ($\pm$0.52\%) demonstrating improved robustness with respect to \hparams. 

\textbf{Simplified recipe.} $\sigma$Reparam also enables a simplified framework for training ViT-B, ViT-L and ViT-H models, in contrast to state-of-the art ImageNet1k ViT training protocols such as
the fully supervised MAE recipe \citep{he2022masked}
and DeiT 
\citep{touvron2021training}, see Table \ref{tab:vision_summary}. In the case of ViT-B models, we are able to train for a shorter duration, remove all pre-LNs layers, remove learning rate (LR) warmup, remove cosine scheduling (requiring only a simple step schedule at 210 epochs) and use no weight decay. Furthermore, $\sigma$Reparam enables SGD training via LARS \citep{you2017large} (with momentum 0.9) -- something not possible with traditional ViT training protocols \citep{touvron2021training,he2022masked}. These simplifications also have the added benefit of reducing GPU memory overhead\footnote{We observe a 8.2\% memory reduction in full fp32 precision (for a 1:1 comparison) with a batch size of 86 per GPU.}. For the ViT-L model we relax the LR schedule back to cosine and slightly increase the training interval to 300 epochs. All models use FP32 precision on the attention and \sigmareparam operands and keep mixed precision training for the rest of the network. The full set of \hparams is available in 
\Cref{appendix:image_cls}. We note that for larger models like the ViT-L/16 and ViT-H/14 a slight weight decay cosine schedule from $0.0$ to $10^{-5}$ enables easier training.

To further understand the effect of \sigmareparam, we track both the attention entropy, and the largest singular value of the attention weight matrix over the course of training. In Figure \ref{fig:joint_sup_vitb_and_21k}, \sigmareparam maintains lower spectral norms for the attention weight matrices and presents a higher, but monotonically decreasing attention entropy throughout training. The benefit of such smooth and bounded attention entropy curves is reinforced by the accelerated performance observed in Test Top 1 and the 50 epoch reduction in training time for the \sigmareparam ViT-B/16 shown in Figure \ref{fig:joint_sup_vitb_and_21k}.

Finally, we extend \sigmareparam to a much larger 11M sample training dataset, ImageNet21k \citep{ridnik2021imagenet21k}, and train a ViT-B/16. We then finetune this model with ImageNet1k and report the performance in Table \ref{tab:imagenet21k}. We observe that \sigmareparam presents competitive results against ViT-B/16's trained on drastically larger datasets such as JFT3B \citep{DBLP:conf/cvpr/Zhai0HB22} and the 400M sample CLIP pre-training dataset \citep{DBLP:journals/corr/abs-2212-06138}, all the while presenting stable training and not requiring LayerNorm or LR warmup.

\subsection{Self-Supervised Training of Visual Representations}
\label{subsec:ssl}
\begin{table*}[h!]
  \caption{\small (Top) Best SimCLR ImageNet1k trial top 1 linear probe performance training for 300 epochs.
  \emph{\sigmareparam + pre-LN} yields the highest performing run, with \emph{Frozen Patcher} performing competitively. (Bottom) Configuration of the variants used in our stability analysis. The MoCo v3 weight initialization and patch initialization scheme are described in \citet{chen2021empirical}. 
  For full hyperparameters, see \Cref{tab:ssl-hparams} of \Cref{sec:app-ssl-hyperparameters}.}
  \label{tab:ssl-variants}
  \centering
  \small
  \begin{tabular}{lcccc}
    \toprule
    & Baseline & Frozen Patcher & \sigmareparam & \sigmareparam + pre-LN\\
    \midrule
    Top 1 @ 300 (ours) & 72.4 & 74.4 & 73.7 & \textbf{74.5} \\
    \midrule
    Weight Init & MoCo v3  & MoCo v3 & \texttt{trunc\_norm(.02)} & \texttt{trunc\_norm(.02)} \\
    Patcher Init     & MoCo v3 & MoCo v3 & \texttt{trunc\_norm(.02)} & \texttt{trunc\_norm(.02)} \\
    Frozen Patcher & No & Yes & No & No \\
    \sigmareparam     & No & No & Yes & Yes \\
    pre-LN     & Yes & Yes & No & Yes \\
    \bottomrule
  \end{tabular}
  
\end{table*}

In computer vision, SSL has been effective in enabling efficient training on downstream tasks
\citep{assran2022masked}. Most of this progress has been made using convolutional architectures, while works using ViTs often  require specialized training recipes \citep{caron2021emerging}.

\begin{figure}[ht!]
    \centering
        \includegraphics[width=0.6\textwidth]{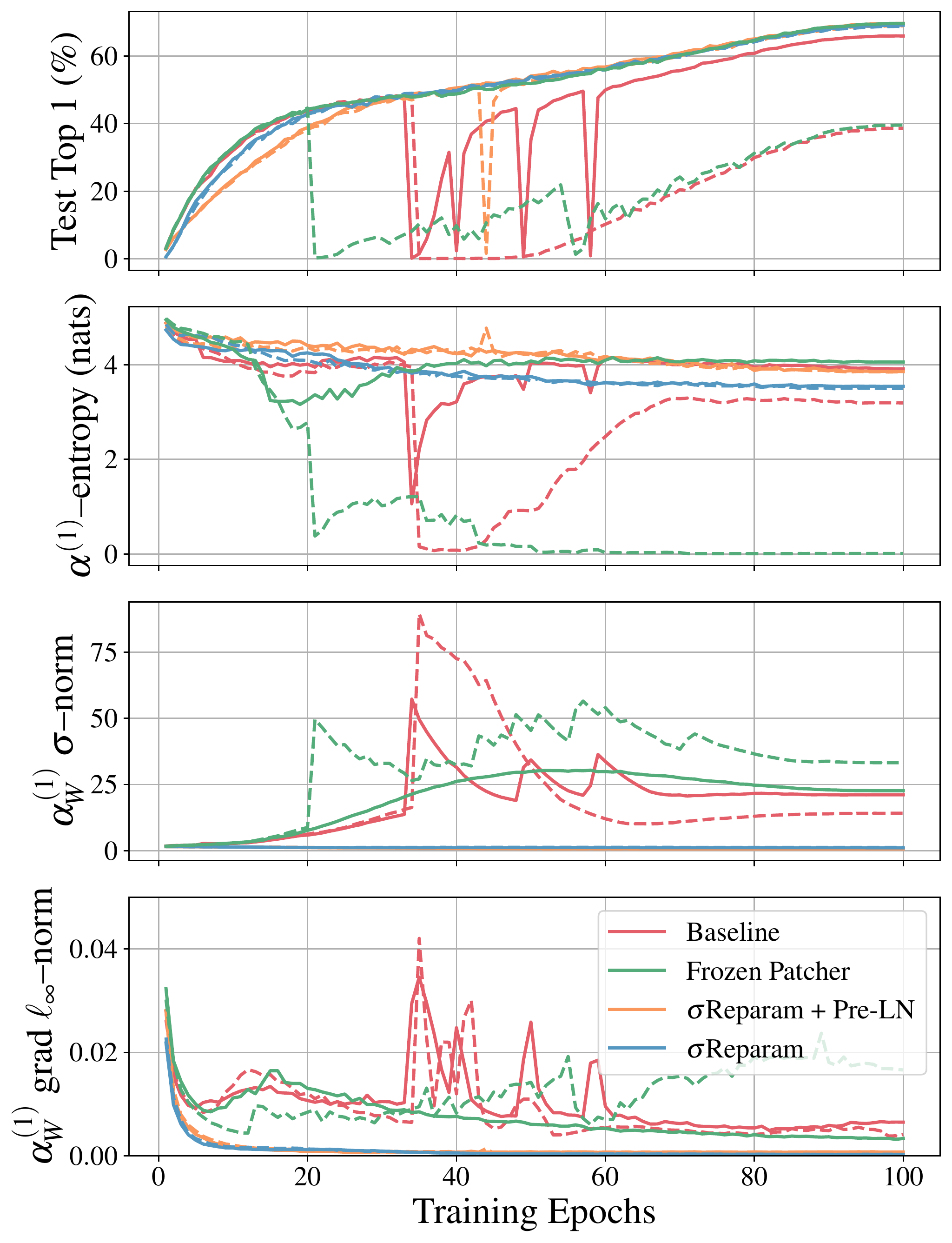}
    \caption{The best (solid line) and worst (dashed line) trials of each method from 10 trials of SimCLR for each method on ImageNet1k with 40 epochs of learning rate warmup. We show classification performance alongside relevant metrics from the first attention layer (top to bottom): attention entropy, the spectral norm of the attention weights, and the $\ell_\infty$--gradient norm of the attention weights.
    We see that the \emph{Frozen Patcher} method functions as intended, regulating its gradient norm, and protecting it from the large gradient norms inducing instability in \emph{Baseline}.
    We also observe a second form of instability during training: the growing spectral norm leads to a poorly behaved attention mechanism, entropy collapse, and a drop in performance as described in \Cref{sec:method}.
    This affects \emph{Baseline}, as well as \emph{Frozen Patcher}, as neither method gives specific protection against this second type of instability (solid and dashed red, and dashed green lines).
    Finally, we see that \sigmareparam with and without pre-LN regulate both the gradient norms, as well as the spectral norms, giving defense against both types of instability.}
    \label{fig:ssl-stats-40}
\end{figure}

Recently, it was found that ViTs suffer from training instabilities in SSL tasks \citep{chen2021empirical}. 
These instabilities can be remedied through a combination of frozen patch embedders, initialization schemes, and longer learning rate warmups; however,
there is an open question whether a general solution providing stable SSL ViT training exists \citep{chen2021empirical}.

Here, we demonstrate that \sigmareparam is a ViT SSL stabilizer.
Taking SimCLR as our SSL method, we investigate four variants. \emph{Baseline} and \emph{Frozen Patcher} were studied in \citet{chen2021empirical}, whereas \emph{\sigmareparam} and  \emph{\sigmareparam + pre-LN} are our solution. 

These methods are detailed in \Cref{tab:ssl-variants}, and their full hyperparameters are given in \Cref{tab:ssl-hparams} of \Cref{sec:app-ssl-hyperparameters}.

\begin{figure}[ht!]
    \centering
        \includegraphics[width=0.6\textwidth]{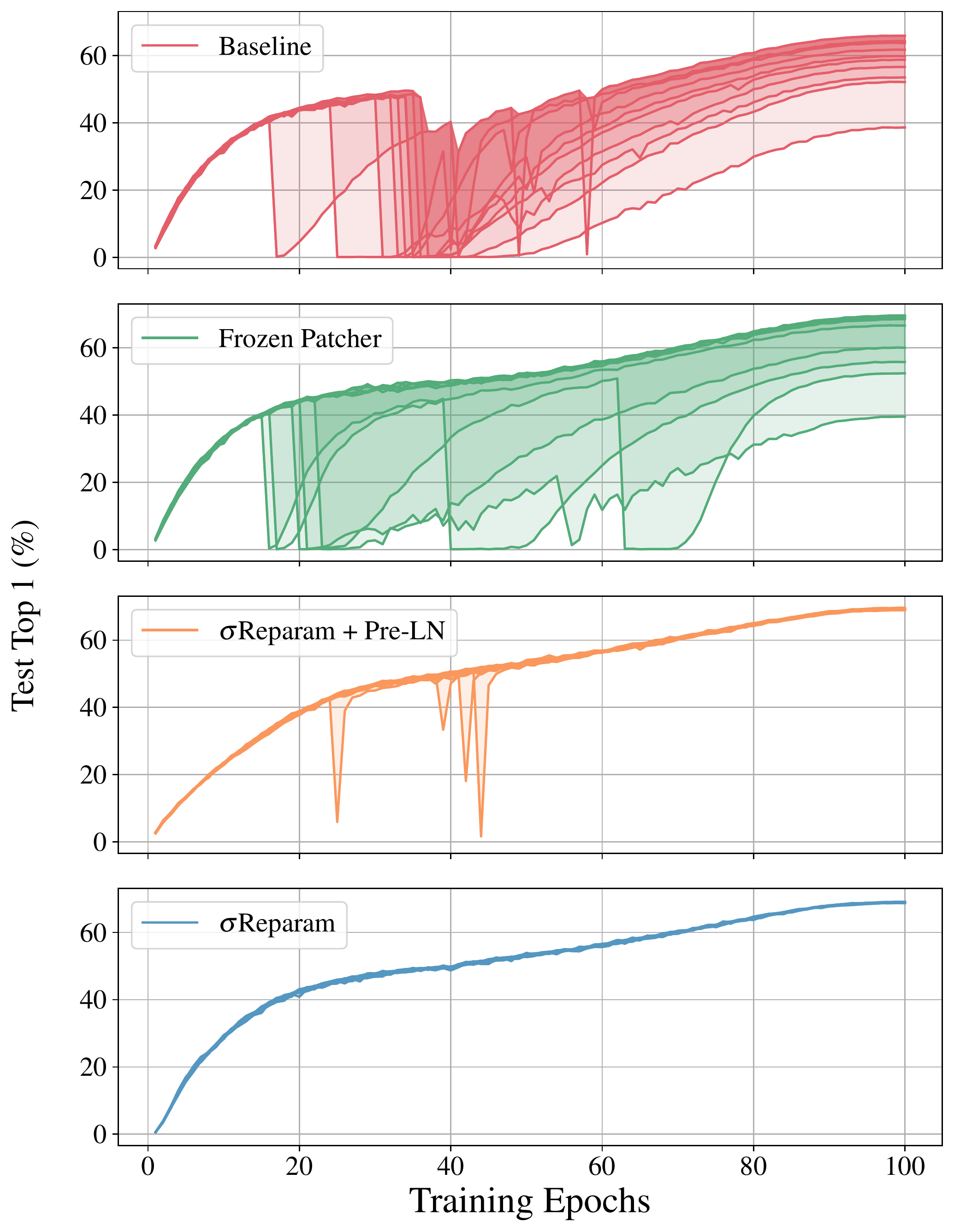}
    \caption{Linear probe performance of each of the 10 trials of SimCLR for each stabilization method.
    We see that \sigmareparam is the most stable method.
    \sigmareparam + pre-LN is also quite stable.
    In the case where it experiences instabilities, we see that it is able to recover much quicker than \emph{Baseline} and \emph{Frozen Patcher}.
    This is due to the regularization of the spectral norm which 1) prevents any arising instability pushing the model too far away from the current solution, and 2) keeps the attention mechanism useful, such that gradients are available for any required correction.}
    \label{fig:ssl-stability-40}
\end{figure}

We observe two types of instability. 
The first, as observed in \citet{chen2021empirical}, is induced by large gradient norms in early layers.
The second, described in \Cref{sec:method}, relates to entropy collapse.
We find that \emph{Frozen Patcher} protects against the first type, but is still susceptible to the second.
\sigmareparam, however, can protect against both types of instability, yielding more reliable training 
(see \Cref{fig:ssl-stability-40,fig:ssl-stats-40}).

\begin{figure}
  \centering
    \includegraphics[width=0.6\textwidth]{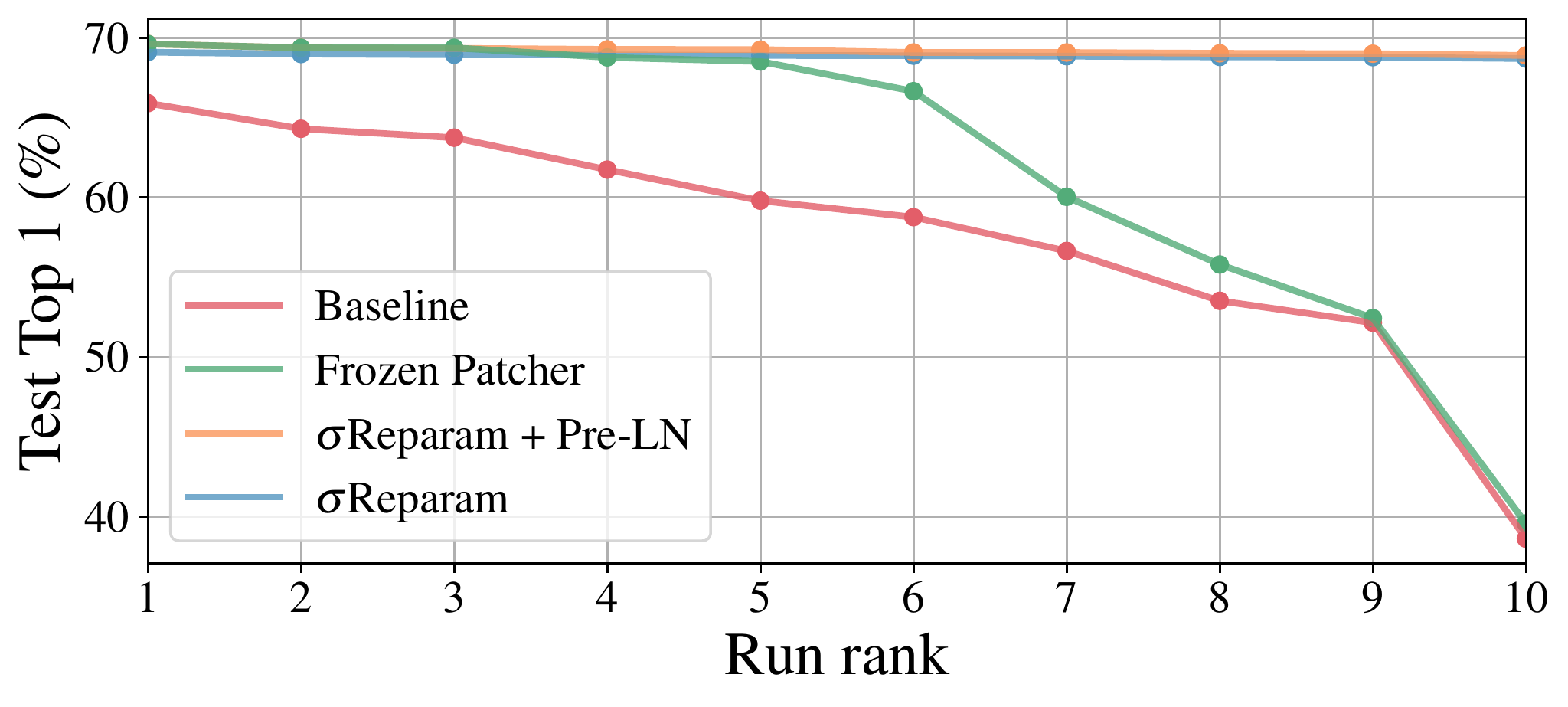}
    \caption{\small Linear probe performance on ImageNet1k at the end of training over 10 trials for each method. 
  Trials are ordered by decreasing performance, with run rank 1 (10) corresponding to the best (worst) trial.
  \emph{Frozen Patcher} and \emph{\sigmareparam + pre-LN} produce the best individual runs, with \emph{\sigmareparam} marginally lower.
  \emph{\sigmareparam + pre-LN} and \emph{\sigmareparam} are the methods most reliably giving good performance, with \emph{Baseline} and \emph{Frozen Patcher} each susceptible to at least one instability type.
  }
  \label{fig:ssl-40-index}
\end{figure}

As noted in \citet{chen2021empirical}, instabilities reduce final performance.
We show the instability impact on performance below.
in \Cref{fig:ssl-40-index}. 
The methods with the best performing individual runs are \emph{Frozen Patcher} and \emph{\sigmareparam + pre-LN}, whereas
the most stable methods are \emph{\sigmareparam + pre-LN} and \emph{\sigmareparam}.

Our main stability experiments use 40 epochs of learning rate warmup, matching the setting of  \citet{chen2021empirical}.
Using \sigmareparam, as in the supervised setting, gives training stability even at the lower learning rate warmup of 10 epochs.
For more details, see \Cref{sec:app-ssl-warmup}.

Finally, we look at the performance attainable when training for a longer duration of 300 epochs in Table \ref{tab:ssl-variants}.
The best performing method run is given by \emph{\sigmareparam + pre-LN}, with \emph{Frozen Patcher} performing almost as well, and both outperforming the reference SimCLR result \citep{chen2021empirical}.

Ultimately, we see while \emph{\sigmareparam} produces the lowest degree of instability, the best overall method for stable training of SimCLR ViTs is \emph{\sigmareparam + pre-LN}, producing both the highest ImageNet1k linear probe performance at 100 epochs (69.6 \%) and 300 epochs (74.5 \%) epochs, as well as very stable training over many trials, both at long and short learning rate warmup.

\subsection{Machine Translation}\label{sec:mt}
\begin{figure}[t!]
    \centering
        \includegraphics[width=0.6\textwidth]{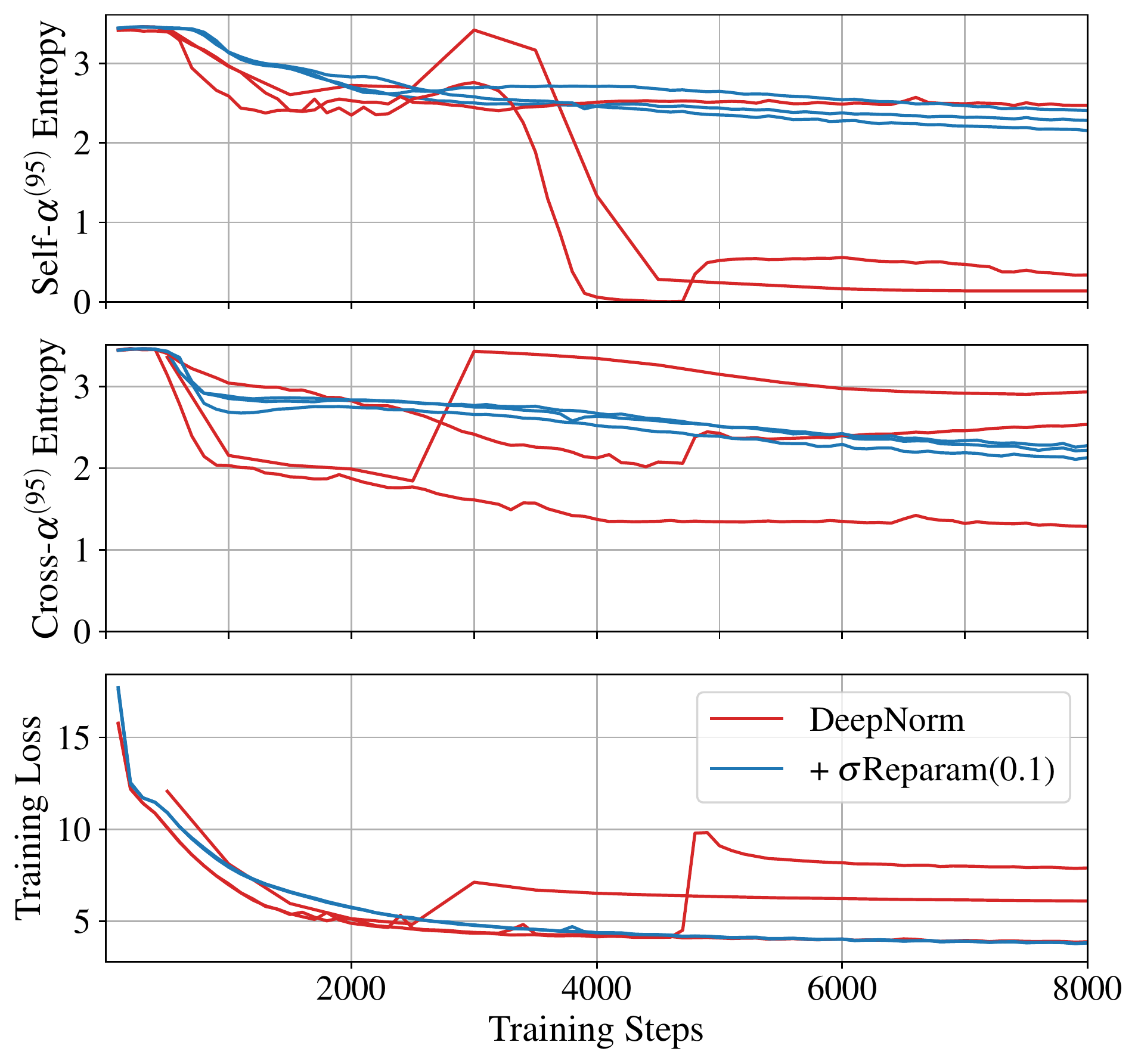}
    \caption{MT training on WMT'17 for 100L-100L DeepNorm and DeepNorm with injected \sigmareparam across 3 runs with different seeds: training loss (bottom), encoder self-attention entropy (top) and encoder-decoder cross-attention entropy (middle) for 95th layers. Attention entropy collapse with further model divergence is observed for DeepNorm, while \sigmareparam is bounding entropy and provides stable training. }
    \label{fig:mt-stability-short}
\end{figure}

In machine translation (MT) stable training of deep encoder-decoder post-LN Transformers is an active research area~\citep{wang2022deepnet,liu2020admin}. Vanishing gradients problem has been reported by many works, leading to different solutions including rescaling residual connections: e.g., \citet{wang2022deepnet} trained a 1000-layer Transformer by properly rescaling residual connections and initialization depending on model depth, dubbed DeepNorm. We examined attention entropy collapse for the deep Transformers in MT and found that they suffer not only from vanishing gradients but also from entropy collapse, both for vanilla post-LN and DeepNorm. By injecting \sigmareparam alongside post-LN/DeepNorm, we empirically show that it is able to bound attention entropy and stabilize training without any divergent training loss growth issues. Details on experiments and all findings are in Appendix~\ref{sec:app:mt}.

\textbf{Empirical setup.} We use standard WMT'17 English-German benchmark with {\it newstest2016} as a validation and {\it newstest2017} as test sets. We consider $N$L-$N$L encoder-decoder models with $N$ encoder and $N$ decoder layers, where $N=6,18,50,100$, for both post-LN and DeepNorm configurations. For all models we report BLEU score on validation and test sets across 3 runs with different seeds. 

\begin{table*}[t!]
    \caption{Results for MT on WMT'17 English-German data for post-LN, with or without additional \sigmareparam, with or without residual rescaling (`DeepNorm' from \citet{wang2022deepnet}). We report average BLEU score and its std across 3 runs with different seeds for a variety of encoder-decoder architectures: 6L-6L, 18L-18L, 50L-50L, and 100L-100L. `DV' states for how many times a model diverges / is not training across runs. With red block we mark unstable baseline training while with blue block -- training stabilized by \sigmareparam.}\label{tab:mt_main}
    \begin{center}
    \setlength\tabcolsep{5pt} 
    \resizebox{\linewidth}{!}{
    \small
    \begin{tabular}{lcccccccccccc}
    \toprule
     \multirow{2}{*}{Models} & \multicolumn{3}{c}{6L-6L} & \multicolumn{3}{c}{18L-18L} & \multicolumn{3}{c}{50L-50L} & \multicolumn{3}{c}{100L-100L} \\
    \cmidrule(lr){2-4} \cmidrule(lr){5-7} \cmidrule(lr){8-10} \cmidrule(lr){11-13}
     & DV & Valid BLEU & Test BLEU & DV & Valid BLEU & Test BLEU & DV & Valid BLEU & Test BLEU & DV & Valid BLEU & Test BLEU \\
    \midrule
    post-LN & 0/3 & 34.2$_{0.2}$ & 27.8$_{0.2}$ & \cellcolor{red!20}{\bf 1/3} & \cellcolor{red!20}35.2$_{0.2}$ & \cellcolor{red!20}29.0$_{0.2}$ & \cellcolor{red!20}{\bf 3/3} & \cellcolor{red!20}- & \cellcolor{red!20}- & \cellcolor{red!20}{\bf 3/3} & \cellcolor{red!20}- & \cellcolor{red!20}- \\
    \,\, + \sigmareparam & 0/3 & 34.3$_{0.3}$ & 27.8$_{0.2}$ & \cellcolor{blue!20}0/3 & \cellcolor{blue!20}35.2$_{0.2}$ & \cellcolor{blue!20}28.7$_{0.2}$ & \cellcolor{blue!20}0/3 & \cellcolor{blue!20}34.9$_{0.3}$ & \cellcolor{blue!20}28.5$_{0.6}$
     & {\cellcolor{red!20}\bf 3/3} & \cellcolor{red!20}- & \cellcolor{red!20}- \\
    DeepNorm & 0/3 & 34.2$_{0.2}$ & 27.9$_{0.2}$ & 0/3 & 35.7$_{0.4}$ & 29.2$_{0.2}$ & 0/3 & 35.7$_{0.2}$ & 29.2$_{0.1}$ & \cellcolor{red!20} {\bf 2/3} & \cellcolor{red!20} 35.2$_{0.0}$ & \cellcolor{red!20} 29.2$_{0.0}$ \\
    \,\, + \sigmareparam & 0/3 & 34.4$_{0.4}$ & 27.7$_{0.2}$ & 0/3 & 35.2$_{0.2}$ & 28.6$_{0.1}$ & 0/3 & 34.8$_{0.4}$ & 28.3$_{0.3}$ & \cellcolor{blue!20} {\bf 0/3} & \cellcolor{blue!20} 34.4$_{0.1}$ & \cellcolor{blue!20} 28.0$_{0.1}$ \\
    \bottomrule
    \end{tabular}
    }
    \end{center}
\end{table*}

\textbf{Attention entropy collapse occurs in deep models.} While we reproduced stable results for 6L-6L post-LN and observed nicely bounded attention entropy behaviour, for 18L-18L configurations, divergence is observed when varying the random seed. By close inspection we observe no vanishing gradients problem, but attention entropy collapse clearly occurs during training. Deeper models, namely 50L-50L and 100L-100L, are unable to train due to vanishing gradients as well as attention entropy collapse for some of the deep layers (\Cref{fig:mt:postln}).
For DeepNorm while we are able to reproduce results for 6L-6L, 18L-18L and 50L-50L depths observing stable training (no any models diverged and training behaved well), yet we observe instability in training of the 100L-100L model, resulting in only 1 over 3 (different seeds) successful run. By closer inspection of the training behaviour we do not see any drastic issue of vanishing gradients, however we see attention entropy collapse, see~\Cref{fig:mt-stability-short,fig:mt:deepnorm}.

\textbf{\sigmareparam resolves entropy collapse in deep models.} To alleviate attention entropy collapse and confirm \sigmareparam effectiveness for deep models we inject \sigmareparam into post-LN and DeepNorm models. As a result, \sigmareparam nicely bounds attention entropy for 18L-18L and 50L-50L post-LN models (\Cref{fig:mt:postln-spectral}), resolving any divergence issues as well as vanishing gradients in the 50L-50L model. \sigmareparam also nicely bounds attention entropy for 18L-18L, 50L-50L, 100L-100L DeepNorm models (\Cref{fig:mt:deepnorm-spectral}), resolving any divergence issues for 100L-100L, see~\Cref{fig:mt-stability-short} (vanishing gradients are not observed as DeepNorm targets it). 
In terms of performance (\Cref{tab:mt_main}), \sigmareparam with post-LN or DeepNorm matches their baselines for 6L-6L and in the same ballpark for 18L-18L. However, \sigmareparam is inferior to DeepNorm for 50L-50L and 100L-100L.

\subsection{Speech Recognition and Language Modeling}
We also conduct empirical analysis of speech recognition in Appendix~\ref{sec:asr} and observe attention entropy collapse for different configurations. \sigmareparam alongside with post-LN (a) stabilizes training of post-LN (b) improves robustness with respect to \hparams and (c) to the best of our knowledge, for the first time allows model training without an adaptive optimizer achieving stable training and comparable performance. For language modeling, see Appendix~\ref{sec:language-modeling}, \sigmareparam simplifies training recipe by removing all LayerNorms and achieves comparable performance to state-of-the-art.

\section{Conclusion}
Transformer training stability is a well acknowledged, but still unsolved problem. This problem comes with many facets, and there are multiple necessary conditions that need to be met in order to guarantee stable and robust training. Our work identifies attention entropy collapse as a unique failure pattern that seems to be commonly observed in a wide range of settings and tasks. We also show that \sigmareparam as a simple reparameterization of the weights can effectively address the entropy collapse problem, which often leads to improved training stability and robustness.

There are also limitations of our work. First of all, it is unclear if there is a causal relationship between entropy collapse and training instability of Transformers. We believe that establishing such a connection will enable a deeper understanding of the challenges of Transformer training from the optimization perspective. Second, \sigmareparam, while effective, is not a panacea. In the practical sense, one might still benefit from combining \sigmareparam with many other useful techniques, including initialization, feature normalization, advanced optimizers, etc. We hope that our work opens new perspectives towards inventing new design and training principles in the future.

\section{Acknowledgement}
We would like to thank Navdeep Jaitly, Vimal Thilak, Russ Webb for their helpful feedback and critical discussions on the experimental part of the work; Samy Bengio, Andy Keller, Russ Webb, Luca Zappella for their help throughout the process of writing this paper; Hassan Babaie, Mubarak Seyed Ibrahim, Li Li, Evan Samanas, Cindy Liu, Guillaume Seguin, Okan Akalin, and the wider Apple infrastructure team for assistance with developing scalable, fault tolerant code; and Shuming Ma for providing details on the DeepNorm reproduction steps. Names are listed in alphabetical order.

\bibliographystyle{natbib.sty}

\clearpage
\appendix
\appendixpage

\section{\texorpdfstring{Proof of \cref{thm:bound} and \cref{prop}}{Lg}}\label{sec:proof}
\spectral*
\begin{proof}
Without loss of generality let $u\in \R^T$ denote the $i$'th row of $A$, $u=A_i$. From the assumptions it holds that $\|u\| \leq \b\sigma$. Let $p = p(u)$ denote the softmax probabilities given by:
\begin{align}
    p_j = \frac{e^{u_j}}{Z},
\end{align}
where $Z = \sum_{k=1}^Te^{u_k}$ is the partition function.
The entropy given $p(u)$ is then:
\begin{align}
    \text{Ent}(u) = -\sum_{j=1}^T \frac{e^{u_j}}{Z} \log \left(\frac{e^{u_j}}{Z}\right) = -\sum_{j=1}^T \frac{u_je^{u_j}}{Z} + \log(Z).
\end{align}

We wish to solve the following minimization problem:
\begin{align}\label{eqn:const}
    \min_u \text{Ent}(u) ~~\text{s.t}~~\|u\|^2 \leq \b\sigma^2,
\end{align}
Define the Lagrangian:
\begin{align}\label{eqn:lagrange}
    \mathcal{L}(u,\lambda) = \text{Ent}(u) + \frac{1}{2}\lambda(\|u\|^2 - \b\sigma^2).
\end{align}

To find all saddle points, we solve the system of equations:
\begin{align}
    \frac{\partial \mathcal{L}(u,\lambda)}{\partial u} = 0,~~~\frac{\partial \mathcal{L}(u,\lambda)}{\partial \lambda} = 0.
\end{align}
Giving rise to the following set of equations:
\begin{align}
    \forall_{1\leq k\leq T},~ \lambda u_k &= \sum_{j=1}^T\frac{e^{u_j}}{Z}\left[\delta_{j,k} - \frac{e^{u_k}}{Z}\right]\left[1 + \log \left(\frac{e^{u_j}}{Z}\right)\right] \label{eqn:lagrange_eqn}\\
    &= p_k[\log(p_k) +\text{Ent}(u)] \label{eqn:lagrange_eqn3}\\
    \|u\|^2 &= \b\sigma^2. \label{eqn:lagrange_eqn2}
\end{align}

As a first step, assume that for the minimizer $u^\star$ of \cref{eqn:const} there exists an index $k^\star$ such that $u^\star_{k^\star} = 0$. Using \cref{eqn:lagrange_eqn3}:
\begin{align}
    0 &= \log(p_k^\star) +\text{Ent}(u) = -\sum_{j=1}^T p_j \log \left(\frac{p_j}{p_{k^\star}}\right) = -\sum_{j=1}^T p_j \log (e^{u_j}) = -\sum_{j=1}^T p_j u_j = -\E u.
\end{align}
From the first set of equations we arrive at the condition:
\begin{align}
    \forall_{u_{j_1} \neq 0,u_{j_2} \neq 0},~p_{j_1}\frac{\log(p_{j_1}) +\text{Ent}(u)}{u_{j_1}} &= p_{j_2}\frac{\log(p_{j_2}) +\text{Ent}(u)}{u_{j_2}}\\
    \longrightarrow p_{j_1} +\frac{\E u}{u_{j_1}} &= p_{j_2} +\frac{\E u}{u_{j_2}}\\
    \longrightarrow p_{j_1} = p_{j_2}.
\end{align}
This however implies that $u_1 = u_2 =...=u_T = 0$, hence a contradiction to \cref{eqn:lagrange_eqn2}.\\
Now, assuming $\forall_{k}~u_k \neq 0$, we have using \cref{eqn:lagrange_eqn3}:
\begin{align}
     \forall_{u_{j_1} \neq u_{j_2}},~&\frac{p_{j_1}}{u_{j_1}}[\log(p_{j_1}) +\text{Ent}(u)] = \frac{p_{j_2}}{u_{j_2}}[\log(p_{j_2}) +\text{Ent}(u)]\\
     &\longrightarrow e^{u_{j_1}}\left(1 - \frac{\E u}{u_{j_1}}\right) = e^{u_{j_2}}\left(1 - \frac{\E u}{u_{j_2}}\right) \label{eqn:equate}.
\end{align}

We now make the following observation: we may assume a solution $u$ to \cref{eqn:const} must contain at least one negative component. To see this, consider $u$ such that $u>0$ component wise, and $\|u\| \leq \b\sigma$. 
We can always move $u$ by some vector $v|\forall_{i,j}~v_i = v_j$ such that $\|u-v\| \leq \b\sigma$ where $u-v$ has at least one negative component. Since all components in $v$ are equal, we have that $\text{Ent}(u) = \text{Ent}(u-v)$. Moreover, without loss of generality we may assume that $\E u >0$ due to the same logic. 

Let $u_{j_1},u_{j_2}<0$, then according to \cref{eqn:equate}:
\begin{align}
e^{u_{j_1}}\left(1 - \frac{\E u}{u_{j_1}}\right) = e^{u_{j_2}}\left(1 - \frac{\E u}{u_{j_2}}\right) >0
\end{align}
Note that $f(x) = e^x(1 - \frac{\gamma}{x})$ is monotonously increasing in $x \in (-\infty,0)$ and $x \in [\gamma,\infty)$ for $\gamma>0$, implying that $u_{j_1} = u_{j_2}$. Similarly, if $u_{j_1}<0$ and $u_{j_2}>0$, then $e^{u_{j_2}}\left(1 - \frac{\E u}{u_{j_2}}\right)>0$ hence $u_{j_2}>\E u_{j_2}$. Since $f(x)= e^x(1 - \frac{\gamma}{x})$ is monotonous in $x$ for both $x<0$ and $x>\gamma$, we conclude that a solution $u$ must contain 2 unique values, one positive and one negative.  
Let the different components be $\alpha,\beta$ such that $\alpha>0,\beta<0$. A minimizer of the  entropy would correspond to a $u$ with $T-1$ components equal to $\beta$, and 1 component equal to $\alpha$, such that: 
\begin{align}
    \alpha = \b\sigma\sqrt{1 - \frac{1}{T}}, ~~~\beta = -\b\sigma\sqrt{\frac{1}{T(T-1)}},
\end{align}
with the corresponding entropy:
\begin{align}
\text{Ent}(u^\star) = \log\left(1 + (T-1)e^{-\b\sigma\sqrt{\frac{T}{T-1}}}\right) + \frac{\b\sigma\sqrt{T(T-1)}e^{-\b\sigma\sqrt{\frac{T}{T-1}}}}{1 + (T-1)e^{-\b\sigma\sqrt{\frac{T}{T-1}}}}.
\end{align}
\end{proof}

\proposition*

\begin{proof}
We have that:
\begin{align}
\sigma(\Delta) \geq \frac{1}{\sqrt{w}}\sqrt{\text{Trace}(\Delta^\top \Delta)} &= \frac{1}{\sqrt{w}}\sqrt{\sum_{i,j=1}^w\frac{\mu_{i,j}^2}{\mu_{i,j}^2 + n_{i,j}^2}} =  \sqrt{w}\sqrt{1 - \frac{1}{w^2}\sum_{i,j=1}^w \frac{n_{i,j}^2}{\mu_{i,j}^2 + n_{i,j}^2}}.
\end{align}
\end{proof}

\clearpage

\section{Relationship Between Entropy Collapse and Training Instability}
\label{sec:app-causal}
\subsection{Experimental Outline}

Here we will investigate the interplay between entropy collapse and training stability by asking: 
\emph{would a model with stable training but not exhibiting entropy collapse have been stable if entropy collapse was induced, all other factors held constant?}
In do-calculus \citep{Pearl09}, this roughly corresponds to checking 
\begin{equation*}
P\left(
\text{stable}=\text{True}|\,
\text{stable}=\text{True},\,
\text{collapse}=\text{False},\,
\text{do}(\text{collapse}=\text{True})
\right)<1.   
\end{equation*}

\paragraph{Inducing entropy collapse} 
Note that logits $\vu\in\R^d$ and temperature $\tau$
give rise to the temperature normalized softmax
\begin{equation}
    p_i(\vu,\tau)=\frac{\exp(u_i/\tau)}{\sum_{j=1}^d\exp(u_j/\tau)}
\end{equation}
and corresponding entropy
\begin{equation}
    H_p(\vu,\tau) = -\frac1d\sum_{i=1}^d p_i(\vu,\tau)\,\log\,p_i(\vu,\tau).
\end{equation}
Holding $\vu$ constant, the entropy is low when 
$\tau\rightarrow 0$, and is high when 
$\tau\rightarrow\infty$.
As entropy collapse is observed in experiments when $H_p(\vu,\tau)\rightarrow 0$, we will attempt to induce entropy collapse by sending $\tau\rightarrow\tau_{\text{target}}$, where $\tau_{\text{target}}\ll1$.

Concretely, for a Transformer model, we normalize the logits of the attention matrix by temperature.
We use the same temperature normalization for every layer, i.e. the Transformer has a \emph{global temperature}.
We start the temperature $\tau=1$ which corresponds to the default Transformer model without temperature normalization.
At a prescribed epoch during training, we perform a \emph{temperature intervention}, where we change the 
temperature from $\tau=1$ to a target temperature $\tau_{\text{target}}$.
The transition is sharp, and happens at the start of the prescribed epoch, which we refer to as the \emph{intervention epoch}.

We use the MAE ViT-B/16 recipe (see \Cref{appendix:image_cls}) for these experiments, and train for a total of 100 epochs on ImageNet1k.
To simplify the analysis, we only use ImageNet1k training augmentations, and use no learning rate decay schedule (i.e. the learning rate is flat after warmup).

\paragraph{Eigenvalues of the Hessian} As properties of the Hessian have been successfully used to gain an understanding of stability of the learning process
\cite{
DBLP:conf/icml/GhorbaniKX19,
DBLP:conf/bigdataconf/YaoGKM20,
cohen2021gradient,
DBLP:journals/corr/abs-2207-14484,
DBLP:journals/corr/abs-2110-04369},
we will also use them in our analysis.
Specifically, we will analyze the magnitude $|\lambda_i|$ of the largest magnitude eigenvalues $\lambda_i$ of the Hessian $H$
\begin{align}
    H_{a,b}
    &
    =\frac{\partial^2 \mathcal L}{\partial \theta^a\partial \theta^b},
    &
    H
    &
    \in
    \R^{P\times P},
    &
    Hv_i
    &=v_i\lambda_i,
    &
    ||v_i||
    &=1,
\end{align}
where $\theta^a$ is the $a$--th parameter, $\mathcal L$ is the scalar loss, $P$ is the number of model parameters, and $v_i$ is the normalized eigenvector corresponding to the eigenvalue $\lambda_i$. 
We take $|\lambda_1| > |\lambda_2| > \ldots > |\lambda_P|$, and call the largest eigenvalue $|\lambda_1|$ the \emph{sharpness}, in line with the stability literature.

Computing and storing the Hessian explicitly is problematic, as it is $O(P^2)$ in time and memory.
Instead, noting that the Hessian Vector Product (HVP) $Hv$ for any vector $v$ can be computed using the Vector Jacobian Product (VJP) or Jacobian Vector Product (JVP), avoiding explicit computation of $H$.
Treating the HVP as a linear operator then allows the use of numerical methods for computing the spectrum
\cite{DBLP:conf/bigdataconf/YaoGKM20,
DBLP:conf/icml/GhorbaniKX19}.
For our iterative method we use the implementation of Lanczos from CuPy \cite{cupy_learningsys2017}.
We compute the 5 largest eigenvalues of $H$ using 32,768 samples from the ImageNet1k  training set, and perform this computation at the end of each training epoch.

\paragraph{The Stability Threshold} 
Different optimization algorithms have a \emph{stability threshold};
under a local quadratic assumption, if any Hessian eigenvalue of the loss exceeds this threshold, iterations of the optimization procedure will diverge
\cite{cohen2021gradient,
DBLP:journals/corr/abs-2207-14484}.
For AdamW, the stability threshold $\Gamma$ is derived in the case of a short time-horizon frozen (i.e. non-fully adaptive) approximation of AdamW, has been shown empirically as a suitable stability threshold for the full algorithm \cite{DBLP:journals/corr/abs-2207-14484}, and is given by
\begin{equation}
    \Gamma = \frac{2 + 2\,\beta_1}{1-\beta_1}\frac1\eta=\frac{38}\eta,
\end{equation}
where $\beta_1=0.9$ is the Adam momentum of the gradient moving average \cite{DBLP:journals/corr/KingmaB14}. 
We include this threshold in our analysis.

\subsection{Results}
\begin{figure*}[h!]
  \centering
  \includegraphics[width=0.9\textwidth]{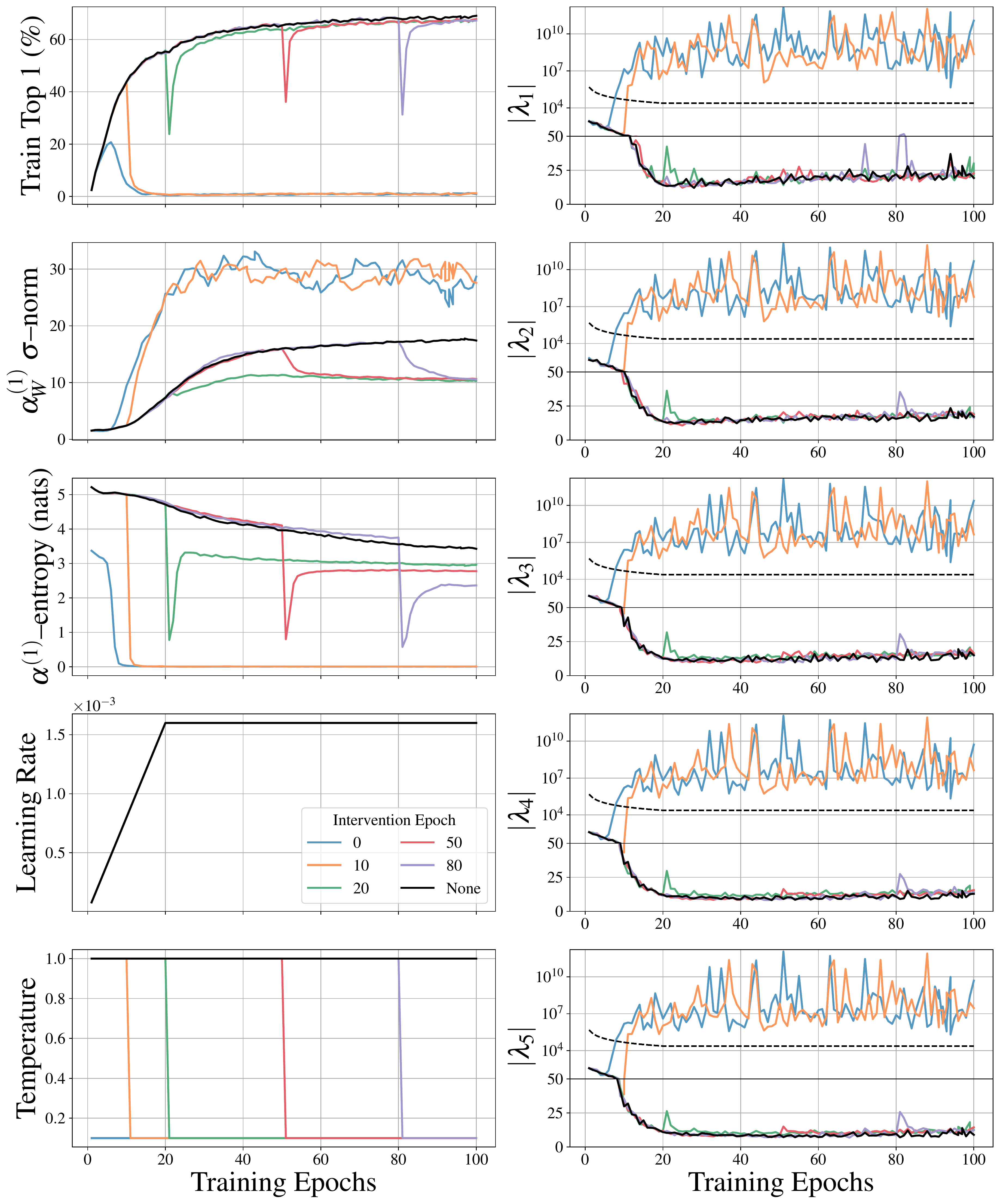}
  \caption{
  Training stability of a Vision Transformer under sharp reductions of its temperature by $10\times$, varying at what epoch in training the intervention occurs.
  We plot (left, top to bottom) training performance, the spectral norm of the first attention projection matrix, the attention entropy of the first attention block, the learning rate and the temperature, (right, top to bottom) the largest to fifth largest singular values of the Hessian by magnitude.
  We see that interventions in the warmup period -- at epochs 10 and 20 -- induce a sharp drop in the entropy $\alpha^{(1)}$ of the attention mechanism in the first Transformer block.
  This is accompanied by an increase in the sharpness $|\lambda_1|$ beyond the stability threshold \cite{cohen2021gradient,
DBLP:journals/corr/abs-2207-14484} (black dashed), resulting in training instability.
Interventions afterwards, at epochs 20, 30, 50 and 80 all induce a drop in attention entropy, but no entropy collapse.
  These models also recover as the sharpness does not exceed the stability threshold.
  We also show the performance of an unintervened Transformer (None).
  }
  \label{fig:causal-temp-start}
\end{figure*}

\begin{figure*}[ht!]
    \centering
    \includegraphics[width=0.9\textwidth]{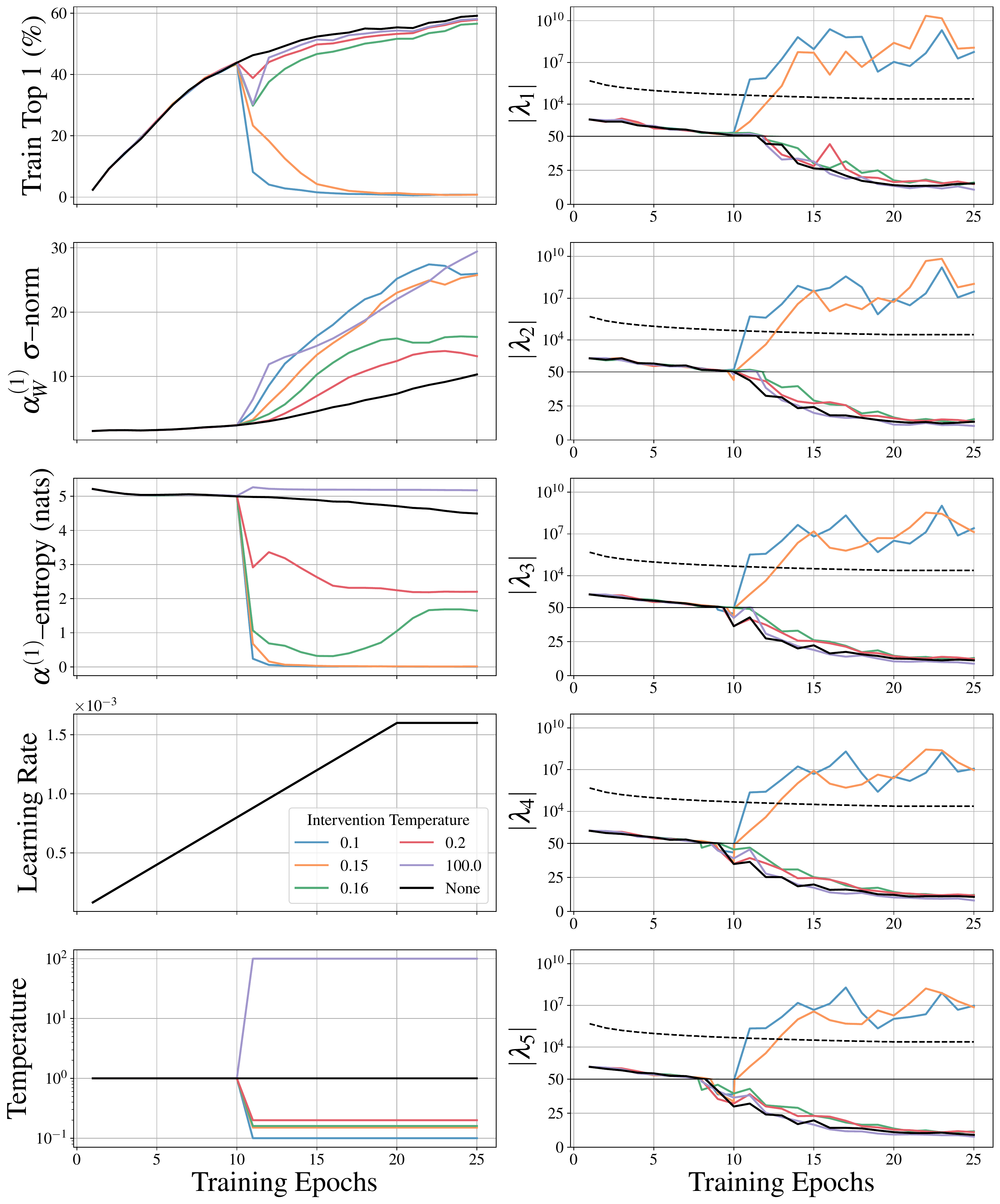}
    \caption{
    Training stability of a Vision Transformer under modifications of its temperature at epoch 10 in training.
    We plot (left, top to bottom) training performance, the spectral norm of the first attention projection matrix, the attention entropy of the first attention block, the learning rate and the temperature, (right, top to bottom) the largest to fifth largest singular values of the Hessian by magnitude.
    We see that reducing the temperature to below 0.15 causes a sharp drop in the entropy $\alpha^{(1)}$ of the attention mechanism in the first Transformer block and an increase in the sharpness $|\lambda_1|$ beyond the stability threshold \cite{cohen2021gradient,
DBLP:journals/corr/abs-2207-14484} (black dashed), resulting in training instability.
Temperatures larger than 0.16 but lower than 1 do not induce training as they do not cross the stability threshold, although these interventions cause a moderate drop in attention entropy before recovery.
We also investigated increasing the temperature, to ensure we were not just ``shocking'' the system, and in fact it is a drop in temperature that is particularly problematic.
Setting the temperature to 100 increases the entropy as expected, but also induces a drop in performance.
    These models also recover as the sharpness does not exceed the stability threshold.
    }
    \label{fig:causal-temp-value}
\end{figure*}

\clearpage
\section{Implementation of \texorpdfstring{\sigmareparam}{Lg}}\label{sec:app:impl}

To compute the spectral norm of the current matrix we use the power method as an approximation method to speed up computations. See Algorithm~\ref{alg:sigmareparam} for a sketch implementation\footnote{By default we use one step of power iteration per gradient update step, similar to~\citep{miyato2018spectral}. Empirically we found no difference in performance when using multiple power iteration steps.}. Note that in practice fp32 precision is typically required for numerical stability. 
We have experimented with various configurations applying \sigmareparam to key and query weights, and/or in other parts (e.g., all other linear layers in the model). While we found that the performance is robust to the configurations, applying it to all the layers amounts to the simplest implementation and also works well in practice, e.g., allowing the removal of LN layers. 
\sigmareparam does not bring any overhead compared to pre-LN or post-LN configurations, see Table~\ref{tab:speed}.

\begin{algorithm}[ht!]
\caption{Pseudo code of $\sigma$Reparam in a PyTorch-like style.}
\label{alg:sigmareparam}
\definecolor{codeblue}{rgb}{0.25,0.5,0.5}
\lstset{
  backgroundcolor=\color{white},
  basicstyle=\fontsize{7.2pt}{7.2pt}\ttfamily\selectfont,
  columns=fullflexible,
  breaklines=true,
  captionpos=b,
  commentstyle=\fontsize{7.2pt}{7.2pt}\color{codeblue},
  keywordstyle=\fontsize{7.2pt}{7.2pt},
}
\begin{lstlisting}[language=python]
# Parameters. W: weight matrix, shape (d, c); gamma: the learned spectral norm, shape (1,) 
# Buffers. u: shape (d,), v: shape (c,), the left and right singular vectors of W
if init: # initialize u, v as random unit vectors and gamma to 1
    u = randn(d)
    u = u / u.norm(dim=0)
    v = randn(c)
    v = v / v.norm(dim=0)
    gamma = ones(1)
if training: # if in the training mode, perform one step of power iteration first
    with torch.no_grad():
        u = W.mv(v)
        u = u / u.norm(dim=0)
        v = W.T.mv(u)
        v = v / v.norm(dim=0)
sigma = einsum('d,dc,c->', u, W, v)
W_hat = gamma / sigma * W # the effective spectral norm of W_hat would be gamma

\end{lstlisting}
\end{algorithm}
\begin{table}[hb]
\caption{Time for one training step for different normalizations in different domains.}\label{tab:speed}
\begin{center}
\small
\setlength\tabcolsep{5pt} 
\begin{tabular}{lcccc}
\toprule
Model & ASR (ms) & MT 8L-18L (ms) \\
\midrule
post-LN & 450 & 1700 \\
pre-LN & 450 & 1800 \\
$\sigma$Reparam & 450 & 2200 \\
\qquad + post-LN & 510 & 2300 \\
\bottomrule
\end{tabular}
\end{center}
\end{table}

\clearpage
\section{Self-Supervised Training of Visual Representations}
\subsection{Hyperparameters}
\label{sec:app-ssl-hyperparameters}

Here we outline the hyperparameters of our experimental setup for SimCLR+ViT stability.
For the variations, alongside their default hyperparameters see \Cref{tab:ssl-hparams}.
These hyperparameters are used in all SimCLR runs unless stated otherwise.

\paragraph{Augmentations} We use SimCLR augmentations throughout, however, we run at half ColorJitter strength, equal to the ColorJitter strength of MoCo v3. 
For completeness, we provide our training augmentation here, our testing augmentation is the standard resize, center crop and normalize.
Half color strength corresponds to \texttt{color\_jitter\_strength = 0.5}.
Setting \texttt{color\_jitter\_strength = 1.0} recovers the base SimCLR training augmentations.

\definecolor{codeblue}{rgb}{0.25,0.5,0.5}
\lstset{
  backgroundcolor=\color{white},
  basicstyle=\fontsize{7.2pt}{7.2pt}\ttfamily\selectfont,
  columns=fullflexible,
  breaklines=true,
  captionpos=b,
  commentstyle=\fontsize{7.2pt}{7.2pt}\color{codeblue},
  keywordstyle=\fontsize{7.2pt}{7.2pt},
}
\begin{lstlisting}[language=python]
[
    transforms.RandomResizedCrop(
        image_size_override, scale=crop_scale, interpolation=Image.BICUBIC
    ),
    transforms.RandomHorizontalFlip(p=0.5),
    transforms.RandomApply(
        [
            transforms.ColorJitter(
                brightness=0.8 * color_jitter_strength,
                contrast=0.8 * color_jitter_strength,
                saturation=0.8 * color_jitter_strength,
                hue=0.2 * color_jitter_strength,
            )
        ],
        p=0.8,
    ),
    transforms.RandomGrayscale(p=0.2),
    transforms.RandomApply([M.GaussianBlur([0.1, 2.0])], p=0.5),
    transforms.ToTensor(),
    IMAGENET_NORMALIZE,
]
\end{lstlisting}

\begin{table}
  \caption{\small Default hyperparameters of the variants of SimCLR used in our stability analysis. 
  The MoCo v3 weight initialization and patch initialization scheme are described in \citet{chen2021empirical}. SinCos refers to stacked 2D SinCos positional encodings \cite{vaswani2017attention}. 
  The table is divided vertically into hyperparameters that differ across methods (top) and hyperparameters shared across methods (bottom).}
  \label{tab:ssl-hparams}
  \centering
  \small
  \begin{tabular}{lcccc}
    \toprule
    & Baseline & Frozen Patcher & \sigmareparam & \sigmareparam + pre-LN\\
    \midrule
    \sigmareparam     & No & No & Yes & Yes \\
    Frozen Patcher & No & Yes & No & No \\    
    Layer Norm     & Yes & Yes & No & Yes \\    
    Patcher Init     & MoCo v3 & MoCo v3 & \texttt{trunc\_norm(.02)} & \texttt{trunc\_norm(.02)} \\    
    Weight Init & MoCo v3  & MoCo v3 & \texttt{trunc\_norm(.02)} & \texttt{trunc\_norm(.02)} \\
    \midrule
    Architecture & ViT-B/16 & ViT-B/16 & ViT-B/16 & ViT-B/16 \\
    Batch Size & 4096 & 4096 & 4096 & 4096\\
    ColorJitter Strength & 0.5 & 0.5 & 0.5 & 0.5 \\    
    Learning Rate & $2\times 10^{-4}$ & $2\times 10^{-4}$ & $2\times 10^{-4}$ & $2\times 10^{-4}$ \\
    Learning Rate Sched & Cosine & Cosine & Cosine & Cosine \\    
    Learning Rate Warmup & 40 Epochs & 40 Epochs & 40 Epochs & 40 Epochs \\    
    Optimizer & AdamW & AdamW & AdamW & AdamW \\
    Positional Encoding & SinCos & SinCos & SinCos & SinCos \\
    Weight Decay & 0.1 & 0.1 & 0.1 & 0.1 \\
    \bottomrule
  \end{tabular}
\end{table}

\begin{figure}[t!]
    \centering
    \begin{subfigure}[b]{0.495\textwidth}
        \centering
        \includegraphics[width=\textwidth]{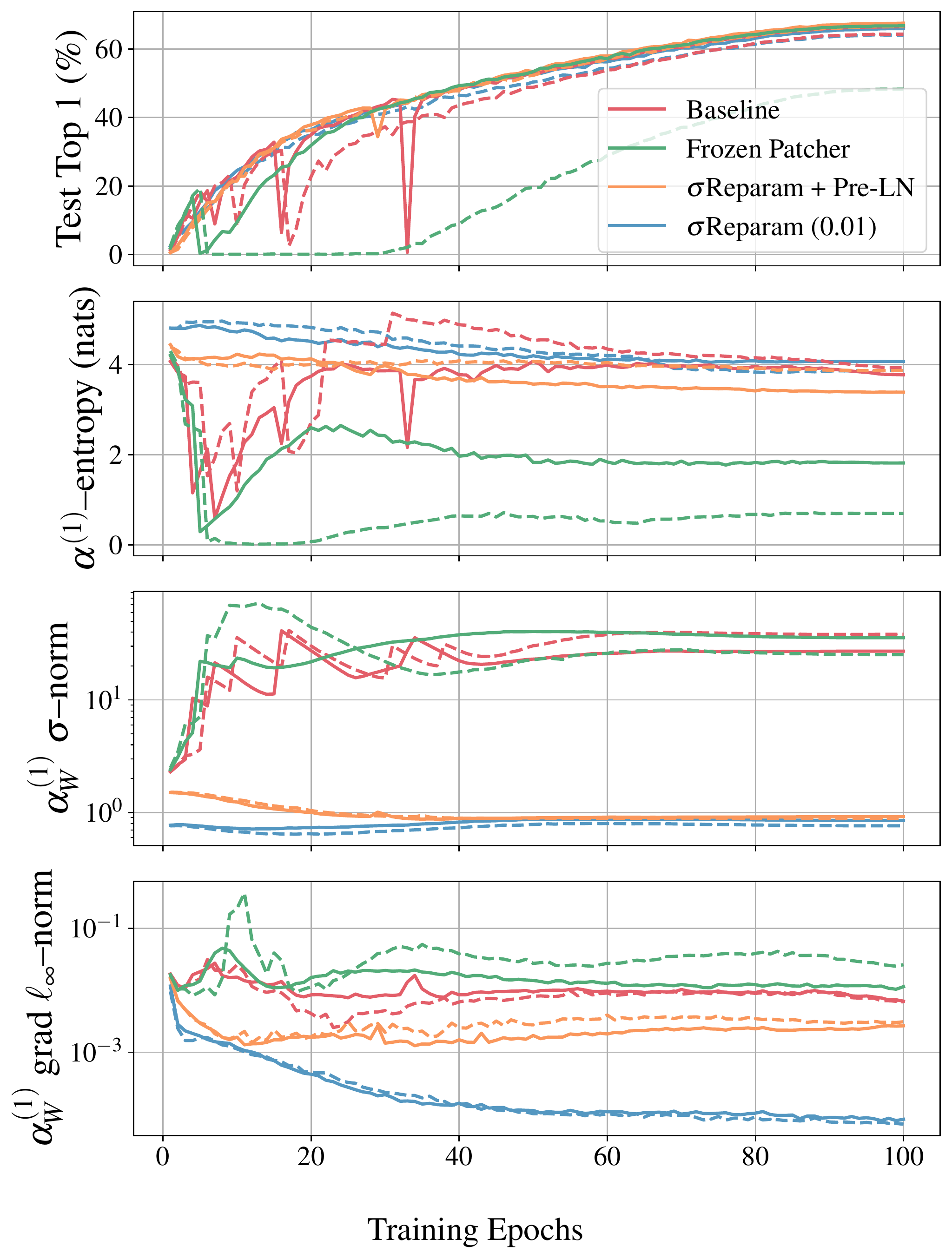}
        \caption{Statistics of best and worst trials per method.}
        \label{fig:ssl-stats-10}
    \end{subfigure}
    \hfill
    \begin{subfigure}[b]{0.495\textwidth}
        \centering
        \includegraphics[width=\textwidth]{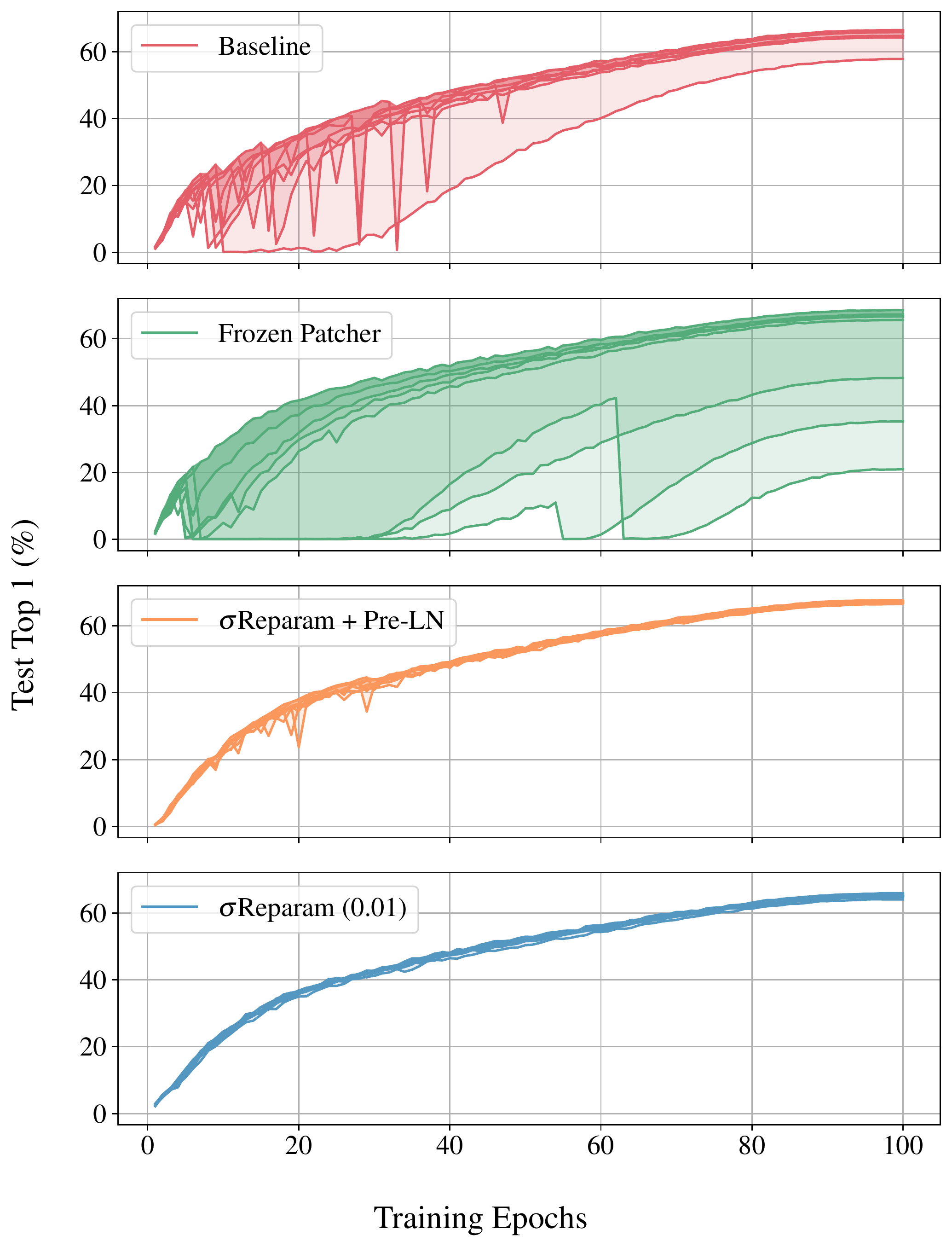}
        \caption{Stability over 8 trials per method.}
        \label{fig:ssl-stability-10}
    \end{subfigure}
    \caption{\small Eight trials of SimCLR for each method on ImageNet1k with 10 epochs of learning rate warmup. \textbf{(a)}
    Linear probe performance for the best (solid line) and worst (dashed line) trials of each method, against relevant metrics from the first attention layer (top to bottom): attention entropy, the spectral norm of the attention weights, and the $\ell_\infty$--gradient norm of the attention weights.
    Our observations are consistent with those of the longer warmup of 40 epochs investigated in \Cref{fig:ssl-stability-40}, except that here, \emph{Frozen Patcher} is less able to tame early layer gradient norms than it was in the longer warmup (dashed green line).
    \textbf{(b)} Linear probe performance of every trial.
    Observations are again consistent with the longer warmup; \sigmareparam with and without pre-LN are the most stable methods.
    \emph{\sigmareparam (0.01)} refers to a \sigmareparam with an initialization scheme of \texttt{trunc\_normal(.01)} instead of \texttt{trunc\_normal(.02)}, with the former showing some signs of instability. Understanding the source of this instability will be the subject of future work. \emph{\sigmareparam + pre-LN} uses the default \texttt{trunc\_normal(.02)}.
    }
    \label{fig:ssl-10}
\end{figure}

\subsection{Reduced Learning Rate Warmup}
\label{sec:app-ssl-warmup}

In \citet{chen2021empirical} the authors noted that the learning rate warmup period needed extending from its typical ImageNet1k default of 10 epochs to 40 epochs, enhancing the stability of the method.
We observe that using \sigmareparam, either with or without pre-LN, we are able to achieve stable SimCLR+ViT training at the original warmup period of 10 epochs. As with our analysis at the longer warmup period, we also investigate the performance distribution across the trials, giving a sense of how instability impacts the final model (see \Cref{fig:ssl-10,fig:ssl-10-index}). 
\begin{figure}
\centering
 \includegraphics[width=0.49\textwidth]{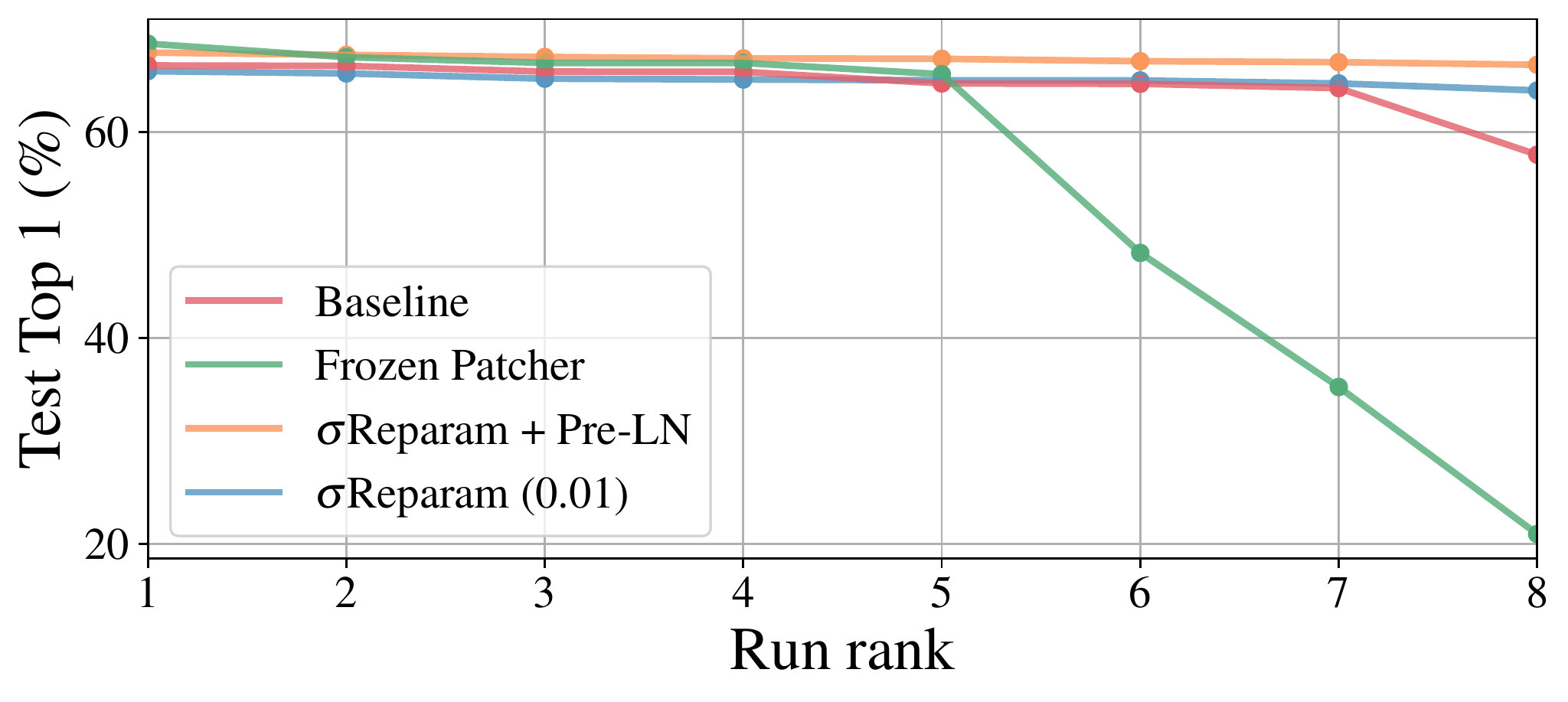}
  \caption{\small Linear probe performance on ImageNet1k at the end of training over 8 trials for each method. 
  Trials are ordered by decreasing performance, with run rank 1 (8) corresponding to the best (worst) trial.
  \emph{Frozen Patcher} produces the best individual, with all other methods marginally lower.
  \emph{\sigmareparam + pre-LN} and \emph{\sigmareparam} are the methods most reliably giving good performance, with \emph{Baseline} and \emph{Frozen Patcher} each susceptible to at least one instability type.
  }
  \label{fig:ssl-10-index}
\end{figure}

\clearpage
\section{Automatic Speech Recognition (ASR)}\label{sec:asr}
In this section we focus on empirical investigation of Transformer training stability and attention entropy collapse phenomenon for automatic speech recognition (ASR) task.

\subsection{Experimental Outline}

{\bf Data}~~~All experiments are performed on the \librispeech{} dataset~\cite{panayotov2015librispeech} where audio paired with transcriptions is available. 
The standard \librispeech{} validation sets (\devclean{} and \devother{}) are used to tune all \hparams, as well as to select the best models. 
Test sets (\testclean{} and \testother{}) are used only to report final word error rate (WER) performance without an external language model.
We keep the original 16kHz sampling rate and compute log-mel filterbanks with 80 coefficients for a 25ms sliding window, strided by 10ms, later normalized to zero mean and unit variance per input sequence.

{\bf Acoustic Model}~~~We stick to a vanilla Transformer model trained with Connectionist Temporal Classification~\citep{graves2006connectionist} loss for simplicity of analysis where only encoder is used (no decoder). We use current, to the best of our knowledge, state-of-the-art vanilla Transformer model configuration and training recipe from~\citet{likhomanenko2020slimipl,likhomanenko2020rethinking}: the model consists of (a) 1D convolution to perform striding (kernel of 7 with stride of 3), (b) Transformer encoder with 36 layers, post-LayerNorm (post-LN), 4 heads, embedding dimension of 768 and MLP dimension of 3072, and (c) a final linear layer to map to the output number of tokens\footnote{The token set consists of the 26 English alphabet letters augmented with the apostrophe and a word boundary token.}. To speed up the model training (2-3x) and decrease memory usage we are using CAPE positional embedding~\citep{likhomanenko2021cape} instead of relative one~\citep{shaw2018self}: both models perform in the same ballpark.

{\bf Training}~~~We follow a training recipe from~\citet{likhomanenko2020slimipl,likhomanenko2020rethinking}. As they, we use SpecAugment~\citep{park2019specaug} which is activated right at the beginning of the training (no difference is found if it is used after 5k training steps): two frequency masks with frequency mask parameter $F=30$, ten time masks with maximum time-mask ratio $p=0.1$ and time mask parameter $T=50$ are used; time warping is not used. We also use Adagrad~\citep{duchi2011adaptive} if not specified otherwise, and learning rate (LR) decaying by 2 each time the WER reaches a plateau on the validation set. We use dynamic batching of 240s audio per GPU and train with tensor cores fp32 on 8 Ampere A100 (40GB) GPUs for 350-500k updates. No weight decay is used. Default warmup is set to 64k and can be varied if stated so. The default LR is 0.03 and is optimized across models. We also apply gradient clipping of 1.

\subsection{Training Stability, Robustness and Generalization}

We start with exploring training stability of the baseline model described above using both pre-LayerNorm (pre-LN) and post-LayerNorm (post-LN) configurations trained on small-scale data, namely 100h of \librispeech{} (\tco).
By varying different \hparams, such as learning rate, warmup, and gradient clipping, post-LN models fail to train. By inspecting the gradient norms per layer and per each parameters' matrix we find a similar vanishing gradients problem as reported, e.g., by~\citet{liu_deep_2020,liu2020admin,wang2022deepnet} for deep Transformers ($>12$ layers) in machine translation domain.
At the same time, pre-LN is stable as reported by, e.g., \citet{nguyen2019transformers,wang2022deepnet,liu2020admin}: we are able to reduce warmup from 64k to 16k, increase learning rate from 0.03 to 0.5, and obtain better results than the training setting from the post-LN baseline. 
However, stable training of pre-LN leads to a degradation in performance compared to post-LN in ASR, similarly as reported in the aforementioned works: validation WER is worse while training loss is lower, see top of Table~\ref{tab:speech-100h}. 
By varying, e.g., learning rate and warmup \hparams and deeper inspecting training stability of pre-LN models we observe that attention entropy is not bounded and can collapse leading to the model divergence with training loss growing, see Figure~\ref{fig:speech:pre-ln-collapse}.

\begin{figure}[t!]
    \centering
    \includegraphics[scale=0.4]{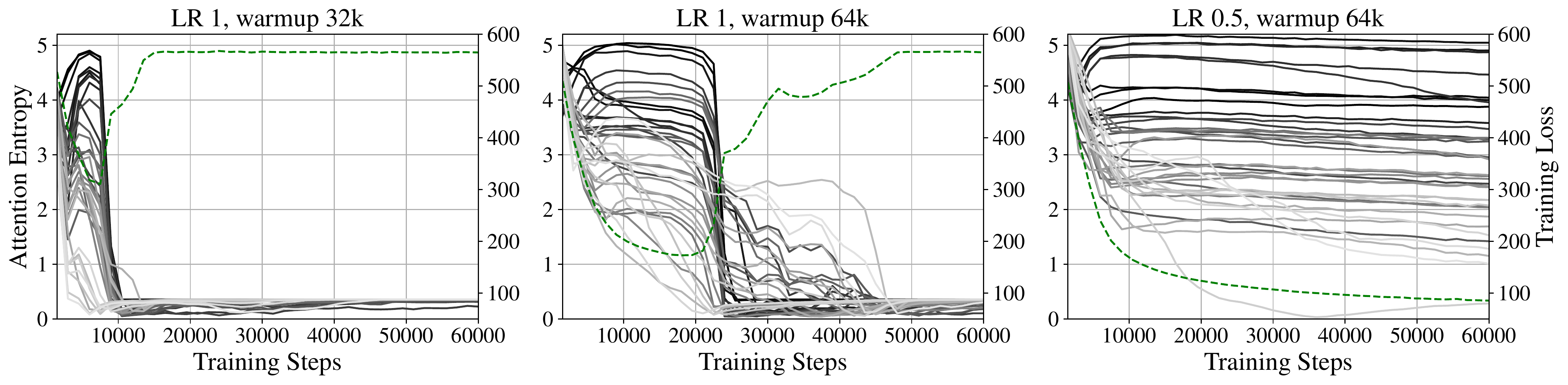}
    \caption{Attention entropy collapse is observed for pre-LN ASR models trained on 100h of \librispeech{} when \hparams, learning rate and warmup, are varied. For every \hparams configuration we plot training loss (dashed, green) and attention entropy for every of 36 layers (solid): a lighter color corresponds to a deeper layer. The right plot (LR 0.5, warmup 64k) gives stable training and the best performance while left (LR 1, warmup 64k) and middle (LR 1 and warmup 32k) have attention entropy collapse phenomenon.}
    \label{fig:speech:pre-ln-collapse}
\end{figure}

\begin{figure}[t!]
    \centering
    \includegraphics[scale=0.345]{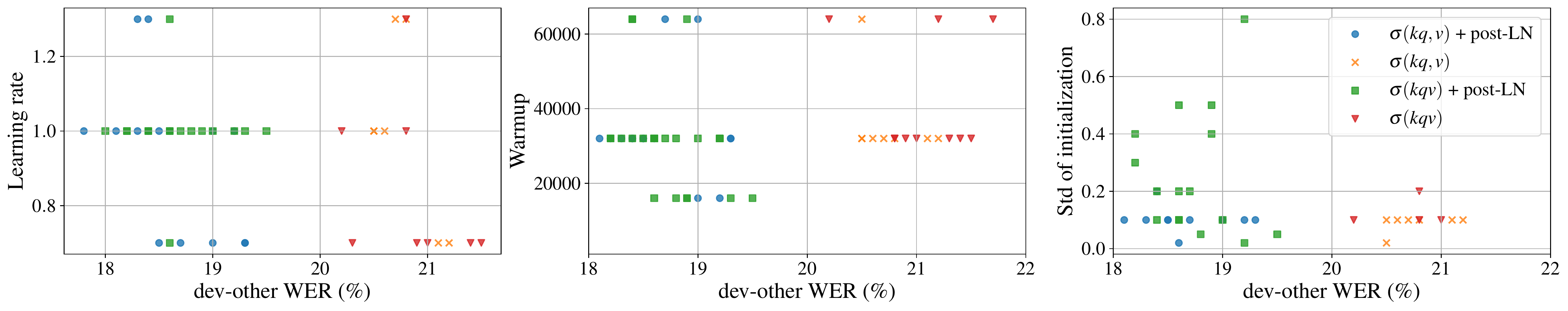}
    \caption{Robustness of \sigmareparam with respect to different \hparams for ASR models trained on 100h of \librispeech{}: learning rate (left), warmup (middle), and initialization std value (right). We report word error rate (WER, x-axis) on the validation {\it dev-other} set.}
    \label{fig:asr_robustness}
\end{figure}

As discussed above in Section~\ref{sec:method}, we now investigate how \sigmareparam affects the training stability and controls the attention entropy bound. 
First, by removing all LayerNorms (pre-LN or post-LN) and switching to \sigmareparam for all linear layers in Transformer blocks and in the final linear layer, we observe (a) stable training similar to pre-LN with no vanishing gradients issue; (b) accepting a wider range of \hparams (Figure~\ref{fig:asr_robustness}) than pre-LN; (c) no attention entropy collapse phenomenon. 
While \sigmareparam significantly outperforms a pre-LN model with the baseline \hparams used for post-LN, it performs worse than an optimized version of a pre-LN model as well as an unstable post-LN model (see top of Table~\ref{tab:speech-100h}). 
However, combining \sigmareparam with post-LN brings two worlds together: stable training similar to pre-LN and generalization similar to post-LN. In summary, {\it \sigmareparam with post-LN achieves (a) similar performance on the validation and test sets and lower training loss (Table~\ref{tab:speech-100h}); (b) no vanishing gradients are observed as for post-LN; (c) the model accepts a wide range of \hparams (Figure~\ref{fig:asr_robustness}) compared to unstable post-LN and stable pre-LN.}

To demonstrate the necessity of \sigmareparam in the form presented in Section~\ref{sec:method}, we compare it with spectral normalization (SN) where $\gamma$ is set to 1 and is not learnable, and WeightNorm~\citep{salimans2016weight} baselines. Both SN and WN perform poorly compared to \sigmareparam (with or without post-LN), see Table~\ref{tab:speech-100h}.

\begin{table}[t!]
\caption{Results for ASR training on 100h of \librispeech{} with \sigmareparam and/or different normalizations: post-layer (post-LN), pre-layer (pre-LN), spectral (SN), weight (WN). We report training loss and word error rate (WER, \% $\downarrow$) for the best models for each configuration: with warmup and Adagrad optimizer (top), and with no warmup and LARS optimizer (bottom). DV states for model divergence. For bottom part: \sigmareparam performs reparametrization for joint matrix for key, queries and values in self-attention, and we are not able to train SN with post-LN configuration.}\label{tab:speech-100h}
\begin{center}
\setlength\tabcolsep{5pt} 
\small
\begin{tabular}{lrrrrrrrrr}
\toprule
 & post-LN & pre-LN & pre-LN & SN & SN & WN & WN & \sigmareparam & \sigmareparam \\
 & & (same) & (optimized) &  & +post-LN & & +post-LN & & +post-LN \\
\midrule
Training loss & 37.7 & 35.3 & 37.2 & 160.4 & 120.3 & 35.6 & 35.4 & 37.5 & 34.9 \\
\midrule
dev-clean WER & 5.9 & 6.9 & 6.2 & 42.6 & 20.3 & 7.0 & 6.3 & 6.4 & 6.1 \\
dev-other WER & 17.7 & 21.3 & 19.1 & 62.9 & 42.7 & 22.3 & 19.4 & 20.5 & 17.8 \\
test-clean WER & 6.2 & 7.1 & 6.3 & 42.4 & 20.4 & 7.3 & 6.7 & 6.8 & 6.4 \\
test-other WER & 17.8 & 21.6 & 19.3 & 63.6 & 43.6 & 22.6 & 19.5 & 21.0 & 18.0 \\
\midrule
\midrule
Training loss & 64.5 & - & 29.4 & 160.0 & DV & 59.1 & 63.2 & 51.1 & 34.2 \\
\midrule
dev-clean WER & 8.1 &  - & 5.9 & 49.8 & DV & 8.3 & 7.1 & 7.2 & 5.8\\
dev-other WER & 25.0 & - &  18.9 & 69.6 & DV & 25.9 & 22.0 & 22.8 & 18.1 \\
test-clean WER & 8.6 &  - & 6.4 & 49.4 & DV & 8.7 & 7.5 & 7.5 & 6.2 \\
test-other WER & 25.6 & - &  19.2 & 70.9 & DV & 26.4 & 22.1 & 23.2 & 18.7 \\
\bottomrule
\end{tabular}
\end{center}
\end{table}

\begin{figure}[t!]
    \centering
    \includegraphics[scale=0.33]{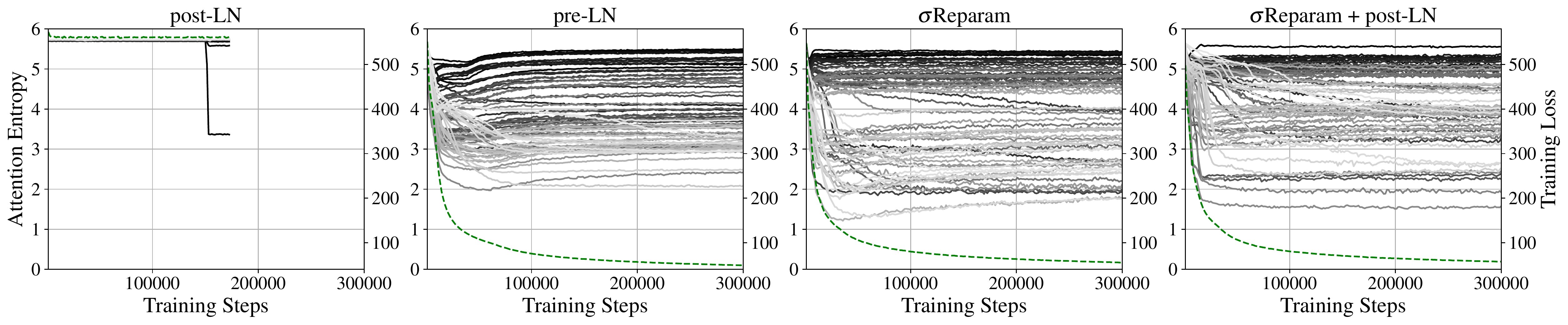}
    \caption{Deep, 72 layers, ASR models trained on 100h of \librispeech{} with different normalizations (from left to right): with post-LN, pre-LN, \sigmareparam, \sigmareparam with post-LN. We plot training loss (dashed, green) and attention entropy for every of 72 layers (solid): a lighter color corresponds to a deeper layer.}
    \label{fig:speech:deep-72}
\end{figure}

\begin{figure}[t!]
    \centering
    \includegraphics[scale=0.39]{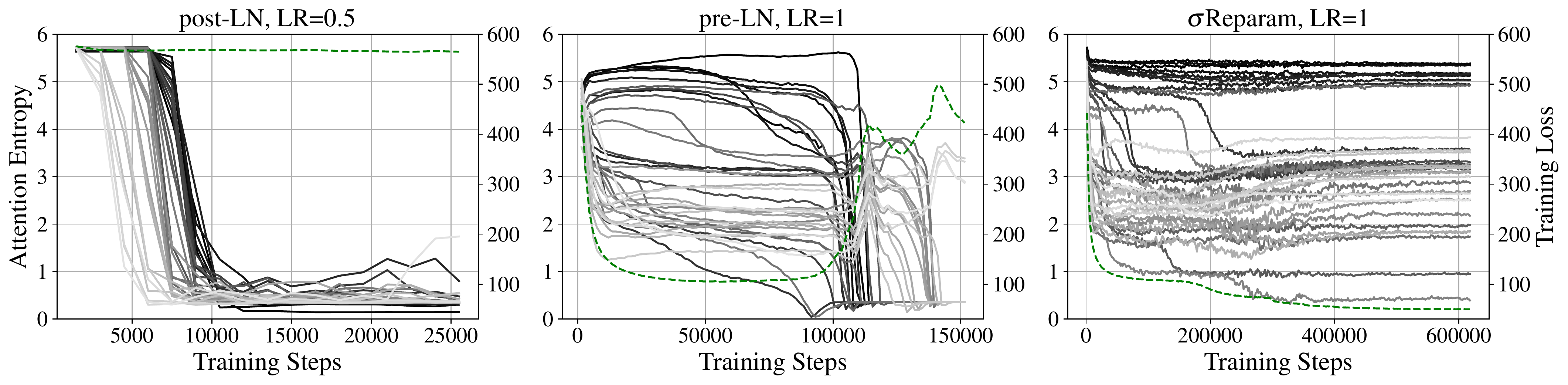}
    \caption{ASR models trained on 100h of \librispeech{} with different normalizations (from left to right: with post-LN, pre-LN, \sigmareparam) and LARS optimizer. We plot training loss (dashed, green) and attention entropy for every of 36 layers (solid): a lighter color corresponds to a deeper layer. Post-LN and pre-LN models have attention entropy collapse when learning rate is increased to 0.5 and 1, correspondingly, while \sigmareparam has no issue.}
    \label{fig:speech:lars}
\end{figure}

We further investigate training behaviour if we increase the model depth by 2x resulting in 72 encoder layers\footnote{The total batch size is reduced by 2x to use the same amount of computational resources.}. In such setting we are unable to train a post-LN model (vanishing gradients are observed) while pre-LN, \sigmareparam and \sigmareparam with post-LN are training out of the box\footnote{Deeper models perform worse compared to smaller ones, however we did not optimize deep models and this is out of scope of the current work.} and have bounded attention entropy throughout the training with no vanishing gradients problem, see Figure~\ref{fig:speech:deep-72}.

\subsection{Training with SGD}\label{sec:asr_lars}
Vanishing gradients and unbalanced gradients can be one of the reasons why the standard SGD fails in training Transformers, especially for deeper architectures, and one needs adaptive optimizers. E.g., \citet{li2022robust} report also another issue with SGD -- ability for generalization, and propose Transformer components modification to improve generalization with SGD training. 

To confirm prior findings, we first experiment with baseline models, pre-LN and post-LN, and SGD optimizer. While post-LN is not training, a pre-LN model can be trained but has a poor generalization. The same holds for \sigmareparam and \sigmareparam with post-LN: the gradient magnitude between the first and last layers can differ not drastically as in post-LN, but generalization is still poor. Similarly to vision experiments, we switch to the LARS \citep{you2017large} (with momentum 0.9) optimizer which normalizes gradients by their magnitudes and thus provides balanced gradients. By carefully tuning only the learning rate from 0.1 to 1.5 (the rest stays the same as for the adaptive optimizer except warmup which is set to 0k) we are able to train pre-LN and post-LN, see bottom of Table~\ref{tab:speech-100h}. 

In our experiments post-LN is more unstable (many learning rates are diverging or not training) and gives significantly worse results than pre-LN. Furthermore, pre-LN is still behind the baseline that uses an adaptive optimizer. 
However, if we switch to $\sigma$Reparam (key, queries and values are represented as one matrix) we observe stable training with respect to learning rate changes, and combined together with post-LN it achieves similar performance to the best results from top of Table~\ref{tab:speech-100h} while keeping the training loss low\footnote{For the separate reparametrization for (keys, queries) and values, we observe less stable training with LARS and no warmup relative to reparametrizing them together.}. {\it To the best of our knowledge, this is the first ASR Transformer model trained without an adaptive optimizer achieving stable training and comparable performance.} Regarding attention entropy collapse, we observe it with LARS training also, see Figure~\ref{fig:speech:lars}: \sigmareparam controls the bound resulting in wider range of accepted \hparams for stable training (models can be trained with learning rate up to 1, while pre-LN and post-LN result in model divergence).

\subsection{\Hparams}

We present \hparams for our ASR experiments on 100h of \librispeech{} in \Cref{tab:speech-hyp}.

\begin{table}[h!]
  \caption{\Hparams comparison for ASR training on 100h of \librispeech{} for models from  Table~\ref{tab:speech-100h}.}
  \label{tab:speech-hyp}
  \centering
  \small

  \begin{tabular}{lccccc}
    \toprule
    & post-LN & pre-LN & \sigmareparam & \sigmareparam + post-LN \\
    \midrule
    dev-clean & 5.9 & 6.2 & 6.4 & 6.1 \\
    dev-other & 17.7 & 19.1 & 20.5 & 17.8 \\
    \midrule
    Weight Init & \texttt{uniform(.036)} & \texttt{uniform(.036)} & \texttt{trunc\_normal(.1)} & \texttt{trunc\_normal(.1)} \\
    \sigmareparam & No & No & Yes & Yes \\
    LayerNorm & Yes & Yes & No & Yes \\ 
    Base LR & 0.03 & 0.5 & 1 & 1 \\
    Optimizer &  \multicolumn{4}{c}{Adagrad}  \\
    LR schedule & \multicolumn{4}{c}{step(330k, 0.5)} \\
    Batch size & \multicolumn{4}{c}{240s x 8}\\
    Weight decay & \multicolumn{4}{c}{none} \\
    Warmup steps & \multicolumn{4}{c}{64k} \\
    Training steps & \multicolumn{4}{c}{500k} \\
    Dropout & \multicolumn{4}{c}{0.3} \\
    Stoch. Depth & \multicolumn{4}{c}{0.3} \\
    SpecAugment & \multicolumn{4}{c}{$F=30$, $T=50$, $p=0.1$, $fmask=2$, $tmask=10$} \\
    Grad. clip & \multicolumn{4}{c}{1} \\
    \midrule
    \midrule
    dev-clean & 8.1 & 5.9 & 7.2 & 5.8 \\
    dev-other & 25 & 18.9 & 22.8 & 18.1 \\
    \midrule
    Weight Init & \texttt{uniform(.036)} & \texttt{uniform(.036)} & \texttt{trunc\_normal(.1)} & \texttt{trunc\_normal(.1)} \\
    \sigmareparam & No & No & Yes & Yes \\
    LayerNorm & Yes & Yes & No & Yes \\
    Base LR & 0.1 & 0.5 & 1 & 0.3 \\
    Optimizer &  \multicolumn{4}{c}{LARS}  \\
    Momentum & \multicolumn{4}{c}{0.9}\\
    LR schedule & \multicolumn{4}{c}{step(300k, 0.5)} \\
    Batch size & \multicolumn{4}{c}{240s x 8}\\
    Weight decay & \multicolumn{4}{c}{none} \\
    Warmup steps & \multicolumn{4}{c}{0k} \\
    Training steps & \multicolumn{4}{c}{500k} \\
    Dropout & \multicolumn{4}{c}{0.3} \\
    Stoch. Depth & \multicolumn{4}{c}{0.3} \\
    SpecAugment & \multicolumn{4}{c}{$F=30$, $T=50$, $p=0.1$, $fmask=2$, $tmask=10$} \\
    Grad. clip & \multicolumn{4}{c}{1} \\
    \bottomrule
  \end{tabular}
\end{table}

\subsection{Large-Scale Experiments: 1k Hours of \librispeech{}}
We also evaluate \sigmareparam for large-scale data: for further experiments we take all $\sim$1k hours of \librispeech{} as the training data. We consider again the Adagrad optimizer with two schedules on learning rate: cosine (with 1 phase of 500k iterations) and step-wise decaying as before for \tco{} experiments. We use exactly the same architecture and \hparams as for small-scale experiments from top of Table~\ref{tab:speech-hyp} except dropout and layer drop which are decreased to 0.1 to decrease model regularization effect. For all models we tune only the learning rate. As before, spectral reparametrization of keys and queries is done separately from values. We also use the learning rate on gamma to be twice bigger than the main learning rate. Similarly to small-scale experiments, training on \librispeech{} shows (see Table~\ref{tab:speech-full-ls}) that {\it $\sigma$Reparam accompanied with post-LN can match the post-LN baseline, while having robustness to the \hparam changes} (e.g. it allows larger learning rate values without any stability issues).

\begin{table}[t!]
\caption{Results for ASR training on full \librispeech{} with \sigmareparam and/or different normalizations: post-layer (post-LN), pre-layer (pre-LN). We report word error rate (WER, \% $\downarrow$) for the best models for each configuration: with step-wise (top) and cosine (bottom) learning rate schedules.}\label{tab:speech-full-ls}
\begin{center}
\setlength\tabcolsep{5pt} 
\small
\begin{tabular}{lrrrrrrrrr}
\toprule
 & post-LN & post-LN & pre-LN & pre-LN & \sigmareparam & \sigmareparam \\
 & \citep{likhomanenko2020rethinking} & & (same) & (optimized) &  & +post-LN \\
\midrule
dev-clean WER & 2.6 & 2.6 & 2.9 & 2.6 & 2.7 & 2.8 \\
dev-other WER & 7.0 & 6.9 & 7.7 & 6.8 & 7.2 & 7.1 \\
test-clean WER & 2.7 & 2.7 & 3.0 & 2.8 & 2.9 & 2.9 \\
test-other WER & 6.9 & 6.9 & 7.8 & 6.8 & 7.3 & 7.0 \\
\midrule
\midrule
dev-clean WER & - & 2.6 & 2.6 & - & 2.8 & 2.7 \\
dev-other WER & - & 7.1 & 6.9 & - & 7.6 & 7.3 \\
test-clean WER & - & 2.9 & 2.8 & - & 3.0 & 2.9 \\
test-other WER & - & 7.2 & 7.0 & - & 7.7 & 7.2 \\
\bottomrule
\end{tabular}
\end{center}
\end{table}

\clearpage
\section{Machine Translation (MT)}\label{sec:app:mt}
In this section we focus on empirical investigation of training stability and attention entropy collapse in deep Transformers for machine translation (MT) with an encoder-decoder architecture. We track attention entropy for the encoder self-attention, the decoder cross-attention and the encoder-decoder self-attention separately to study the entropy collapse phenomenon. The goal of this section is to understand {\it how varying the model depth for the well-established recipes affects the training stability}.

\subsection{Experimental Outline}

We build our experiments on top of the open-sourced code\footnote{\url{https://github.com/microsoft/torchscale}} and baseline recipes provided by \citet{wang2022deepnet}. We follow their instructions\footnote{\url{https://github.com/microsoft/torchscale/tree/main/examples/fairseq\#example-machine-translation}} and \hparams given in \citet{wang2022deepnet}.

{\bf Data}~~~Following~\citet{wang2022deepnet} we perform all experiments on standard WMT'17 English-German benchmark\footnote{\url{https://www.statmt.org/wmt17/translation-task.html}}: we use all provided training data for English-German pair, {\it newstest2016} set as a validation set and {\it newstest2017} as a test set for final evaluation purpose only. We use Fairseq~\citep{ott2019fairseq} script to preprocess data: it uses Byte Pair Encoding (BPE) vocabulary jointly for source and target language  resulting in 41k subword tokens.

{\bf Models}~~~We consider both regular and deep configurations for a vanilla encoder-decoder Transformer model with $N$ encoder and $N$ decoder layers where $N$ is taken as 6 (6L-6L), 18 (18L-18L), 50 (50L-50L), and 100 (100L-100L). Every Transformer layer in each configuration has an embedding dimension of 512, MLP dim of 2048, and 8 heads. Sinusoidal absolute positional embedding~\citep{vaswani2017attention} is used for both encoder and decoder. 

{\bf Training}~~~ We strictly follow the same training recipe from~\citet{wang2022deepnet} (without using back-translation or other domain-specific augmentations) with detailed \hparams in \Cref{tab:mt-hyp}. All models are trained on 8 GPUs of A100 80GB with mixed precision computations and dynamic batching resulting in total batch size of 524288 tokens: for each architecture we pack maximum tokens per GPU and use gradient accumulation (4 for 6L-6L and 18L-18L, 8 for 50L-50L and 16 for 100L-100L).

\begin{table}[h!]
  \caption{\Hparams comparison for MT training on WMT'17 for models from  Table~\ref{tab:mt}.}
  \label{tab:mt-hyp}
  \centering
  \small
  \begin{tabular}{lccc}
    \toprule
    & pre-LN/post-LN/DeepNorm & \sigmareparam + post-LN & \sigmareparam + deepnorm \\
    \midrule
    Weight Init & Fairseq & \texttt{trunc\_normal(.1/.01)} & \texttt{trunc\_normal(.1/.01)} \\
    \sigmareparam & No & Yes & Yes \\
    LayerNorm & Yes & Yes & Yes \\ 
    Base LR & 1.4e-3 & 4.5e-3 & 4.5e-3 \\
    Optimizer &  \multicolumn{3}{c}{Adam}  \\
    LR schedule & \multicolumn{3}{c}{inverse sqrt} \\
    Batch size & \multicolumn{3}{c}{4096 tokens x 8 GPUs x 16 gradient accumulation}\\
    Weight decay & \multicolumn{3}{c}{0.0001} \\
    Warmup steps & \multicolumn{3}{c}{4k} \\
    Warmup init LR & \multicolumn{3}{c}{1e-7} \\
    Training steps & \multicolumn{3}{c}{100k} \\
    Dropout & \multicolumn{3}{c}{0.4} \\
    Grad. clip & \multicolumn{3}{c}{0} \\
    Adam $\epsilon$ & \multicolumn{3}{c}{1e-8} \\
    Adam $\beta$ & \multicolumn{3}{c}{(0.9, 0.98)} \\
    Label smoothing  & \multicolumn{3}{c}{0.1} \\
    \bottomrule
  \end{tabular}
\end{table}

{\bf Evaluation}~~~As it is not specified in~\citet{wang2022deepnet} how the best checkpoint is selected on the validation set, we decided to stick to simple rule: checkpoint with best perplexity on the validation set is selected and further evaluated on both validation and test sets for BLEU score computation which is reported throughout the paper. BLEU is computed by in-built BLEU scripts of Fairseq with the beam size of 5. As reported in prior works we also observe a strong correlation between perplexity and BLEU score: improved perplexity leads to better BLEU score. However BLEU scores on validation and test sets are less correlated and high variation is observed. For that reason we often perform 3 runs with different seeds to estimate standard deviation (std) of the BLEU score.

\subsection{Training Stability of Deep Models}

\begin{table}[t!]
    \caption{Results for MT on WMT'17 English-German data for post-LN, with or without additional \sigmareparam, with or without residual rescaling (`DeepNorm' from \citet{wang2022deepnet}). We report average BLEU score and its std across 3 runs with different seeds for a variety of encoder-decoder architectures: 6L-6L, 18L-18L, 50L-50L, and 100L-100L. `DV' states for how many times a model diverges / is not training across runs. With red block we mark unstable baseline training while with blue block -- training stabilized by \sigmareparam.}\label{tab:mt}
    \begin{center}
    \setlength\tabcolsep{5pt} 
    \resizebox{\linewidth}{!}{
    \small
    \begin{tabular}{lcccccccccccc}
    \toprule
     \multirow{2}{*}{Models} & \multicolumn{3}{c}{6L-6L} & \multicolumn{3}{c}{18L-18L} & \multicolumn{3}{c}{50L-50L} & \multicolumn{3}{c}{100L-100L} \\
    \cmidrule(lr){2-4} \cmidrule(lr){5-7} \cmidrule(lr){8-10} \cmidrule(lr){11-13}
     & DV & Valid BLEU & Test BLEU & DV & Valid BLEU & Test BLEU & DV & Valid BLEU & Test BLEU & DV & Valid BLEU & Test BLEU \\
    \midrule
    pre-LN & 0/3 & 34.2$_{0.1}$ & 27.4$_{0.1}$ & 0/3 & 35.3$_{0.1}$ & 28.8$_{0.1}$ & 0/3 & 34.9$_{0.1}$ & 28.5$_{0.1}$ & 0/3 & 34.7$_{0.1}$ & 28.3$_{0.1}$ \\
    \midrule
    post-LN & 0/3 & 34.2$_{0.2}$ & 27.8$_{0.2}$ & \cellcolor{red!20}{\bf 1/3} & \cellcolor{red!20}35.2$_{0.2}$ & \cellcolor{red!20}29.0$_{0.2}$ & \cellcolor{red!20}{\bf 3/3} & \cellcolor{red!20}- & \cellcolor{red!20}- & \cellcolor{red!20}{\bf 3/3} & \cellcolor{red!20}- & \cellcolor{red!20}- \\
    \,\, + \sigmareparam & 0/3 & 34.3$_{0.3}$ & 27.8$_{0.2}$ & \cellcolor{blue!20}0/3 & \cellcolor{blue!20}35.2$_{0.2}$ & \cellcolor{blue!20}28.7$_{0.2}$ & \cellcolor{blue!20}0/3 & \cellcolor{blue!20}34.9$_{0.3}$ & \cellcolor{blue!20}28.5$_{0.6}$
     & {\cellcolor{red!20}\bf 3/3} & \cellcolor{red!20}- & \cellcolor{red!20}- \\
     \midrule
    DeepNorm & 0/3 & 34.2$_{0.2}$ & 27.9$_{0.2}$ & 0/3 & 35.7$_{0.4}$ & 29.2$_{0.2}$ & 0/3 & 35.7$_{0.2}$ & 29.2$_{0.1}$ & \cellcolor{red!20} {\bf 2/3} & \cellcolor{red!20} 35.2$_{0.0}$ & \cellcolor{red!20} 29.2$_{0.0}$ \\
    \,\, + \sigmareparam & 0/3 & 34.4$_{0.4}$ & 27.7$_{0.2}$ & 0/3 & 35.2$_{0.2}$ & 28.6$_{0.1}$ & 0/3 & 34.8$_{0.4}$ & 28.3$_{0.3}$ & \cellcolor{blue!20} {\bf 0/3} & \cellcolor{blue!20} 34.4$_{0.1}$ & \cellcolor{blue!20} 28.0$_{0.1}$ \\
    \bottomrule
    \end{tabular}
    }
    \end{center}
\end{table}
    
We start with exploring training stability of the baseline model described in~\citet{wang2022deepnet} with pre-LayerNorm (pre-LN) and post-LayerNorm (post-LN) across different depths (all \hparams stay the same except depth is varied). Note that post-LN is a popular design choice for MT tasks due to its good generalization properties. 

For pre-LN models, we reproduced stable results and convergence, however the BLEU score we get is better (\Cref{tab:mt}) than reported by~\citet{wang2022deepnet}. 
We also observed the same trend of decreasing model performance with increasing the model depth. 
Attention entropy is nicely bounded across all depths similarly to ASR\footnote{Note, we did not do any \hparams search to investigate how models behave with, e.g., wider range of learning rates as we did for ASR models.}, see \Cref{fig:mt:preln}.

\begin{figure}[h!]
    \centering
    \includegraphics[scale=0.33]{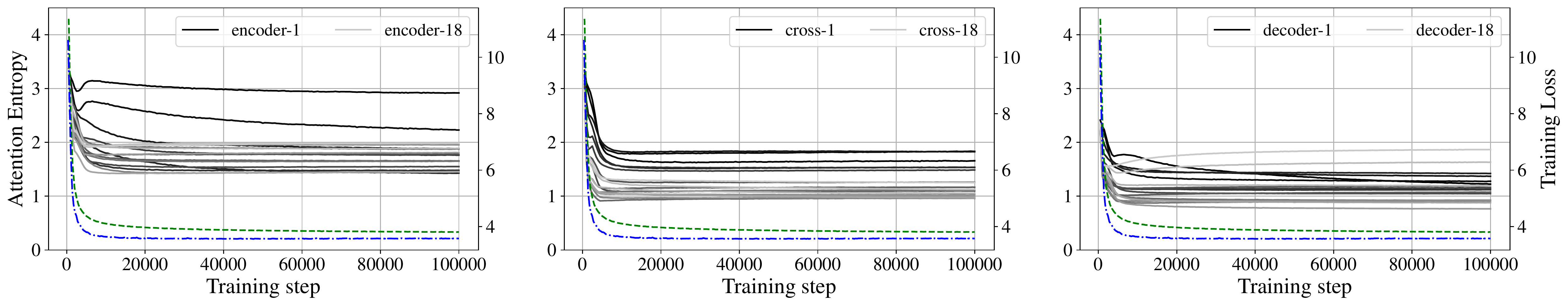}
    \includegraphics[scale=0.33]{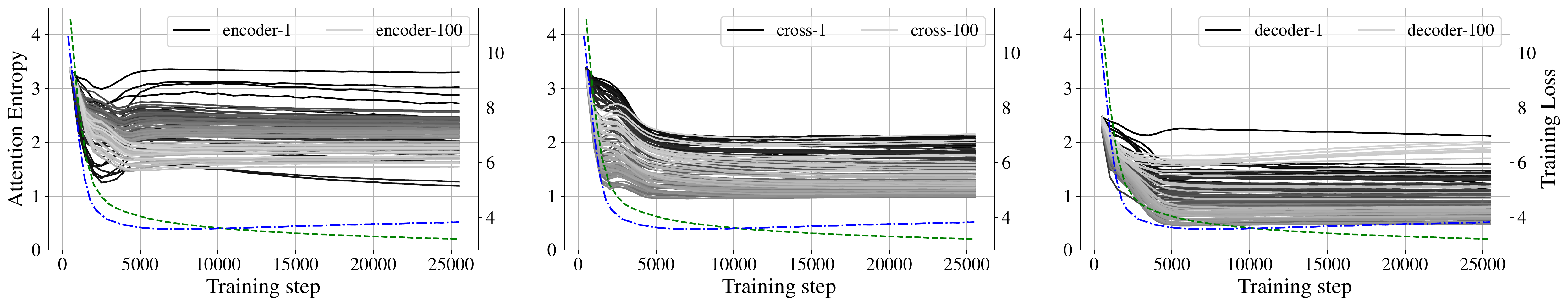}
    \caption{Attention entropy behaviour for MT models trained on WMT'17 with pre-LN for 18L-18L (top) and 100L-100L (bottom): encoder self-attention (left), encoder-decoder cross-attention (middle) and decoder self-attention (right). We plot training (dashed, green) and validation (dot-dashed, blue) losses and attention entropy across all Transformer layers (solid): a lighter color corresponds to a deeper layer. Both deep and shallow pre-LN models have nicely bounded attention entropy and no instability issues are observed across runs with different seeds.}
    \label{fig:mt:preln}
\end{figure}

For post-LN models, we reproduced stable results for 6L-6L depth and observe nicely bounded attention entropy behaviour. However for 18L-18L configurations, divergence is observed when varying the random seed. By close inspection we observe no vanishing gradients problem while attention entropy collapse clearly occurs during training (compare top and middle in~\Cref{fig:mt:postln}) in the encoder attention and the encoder-decoder cross-attention. Deeper models, namely 50L-50L and 100L-100L, are unable to train and we observe the same vanishing gradients problem as reported by~\citet{wang2022deepnet,liu2020admin} as well as attention entropy collapse for some of the deep layers across the board, see bottom plot in \Cref{fig:mt:postln}.

\begin{figure}[h!]
    \centering
    \includegraphics[scale=0.33]{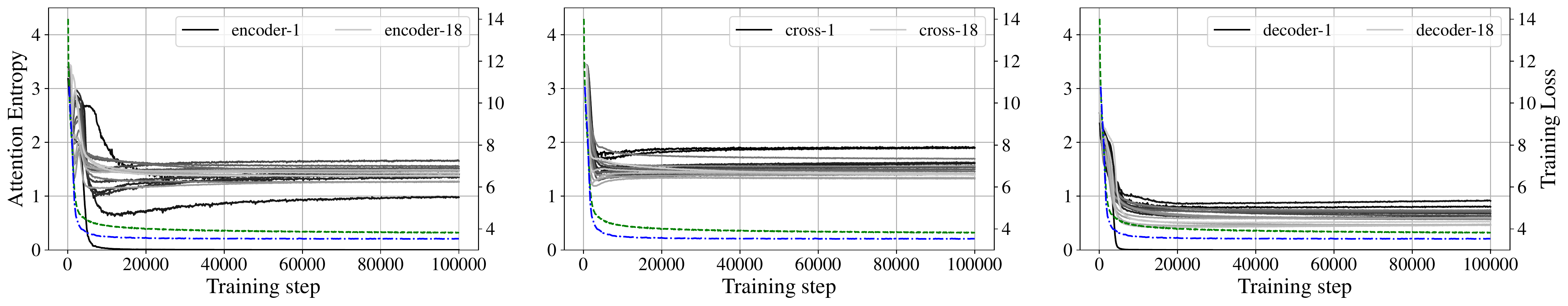}
    \includegraphics[scale=0.33]{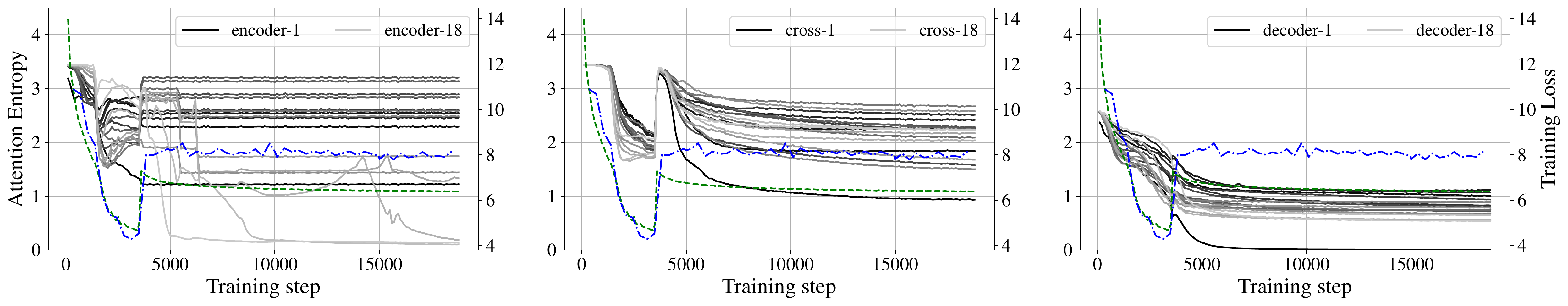}
    \includegraphics[scale=0.33]{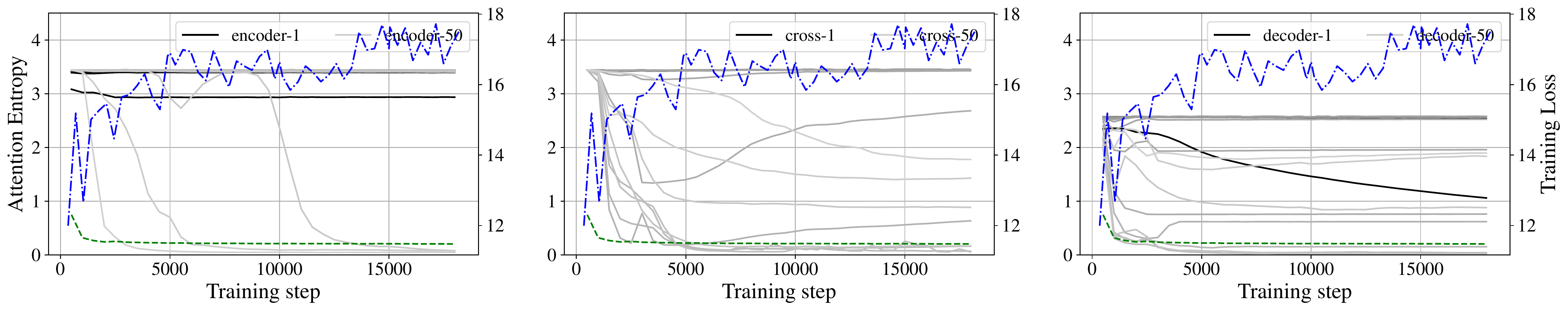}
    \caption{Attention entropy behaviour for MT models trained on WMT'17 with post-LN for 18L-18L (top, middle) with two seeds and 50L-50L (bottom): encoder self-attention (left), encoder-decoder cross-attention (middle) and decoder self-attention (right). We plot training (dashed, green) and validation (dot-dashed, blue) losses and attention entropy across all Transformer layers (solid): a lighter color corresponds to a deeper layer. While 6L-6L (not shown in Figure) has stable training, deeper models experience different issues in training: 18L-18L can be stable (top) or can diverge with attention entropy collapse phenomenon (middle) for the same \hparams but different seed, while 50L-50L has vanishing gradients and layers are not training resulting in constant attention entropy.}
    \label{fig:mt:postln}
\end{figure}

\citet{wang2022deepnet, liu2020admin} are recent works that proposed to rescale residual connections. To stabilize training and resolve vanishing gradients problem in deep post-LN models to preserve post-LN generalization properties. We focus in this paper on~\citet{wang2022deepnet}, DeepNorm, solution (it uses post-LN and rescale residual connections depending on the initial model depth) as they reported ability to train up to 1000-depth Transformer models. We are able to reproduce DeepNorm results for 6L-6L, 18L-18L and 50L-50L depths observing stable training (no any models diverged and training went nicely). However we see no  performance gain of a 50L-50L depth model over a 18L-18L model. Furthermore, we observe instability in training of the 100L-100L model resulting in only 1 successful run among 3 (only seed is varied) while 2 others are diverging after some time (training loss is growing). By close inspection of the training behaviour we do not see any drastic issue of vanishing gradients, however we see the attention entropy collapse happening, see \Cref{fig:mt:deepnorm}. First of all, attention entropy is not bounded for DeepNorm even in 18L-18L and 50L-50L similarly to what we observed in post-LN models. Also a tiny attention entropy collapse happens in 50L-50L (see top plot in~\Cref{fig:mt:deepnorm}) though it does not lead to any divergence. Second, attention entropy collapse is clearly pronounced for 100L-100L models (second, third, and forth rows of~\Cref{fig:mt:deepnorm}) leading  to 2/3 seeds divergence and one with worse performance than 50L-50L models\footnote{From our empirical observations in other domains it could be that deeper models are worse as any attention entropy collapse degrades optimization process resulting in worse generalization.}. Finally, it is interesting to note that attention entropy collapse in 100L-100L can happen for different layers, first and / or last, and with different regimes for the encoder/decoder self-attention and the encoder-decoder cross-attention.

\begin{figure}[t!]
    \centering
    \includegraphics[scale=0.33]{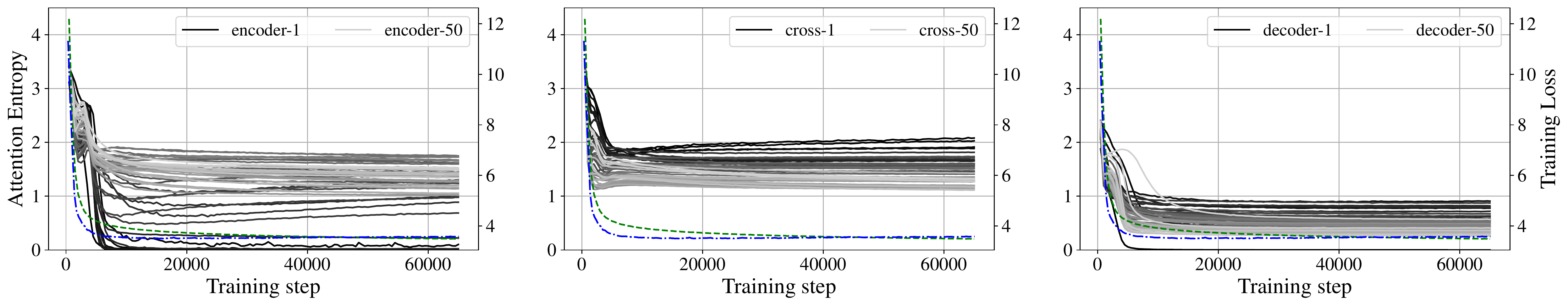}
    \includegraphics[scale=0.33]{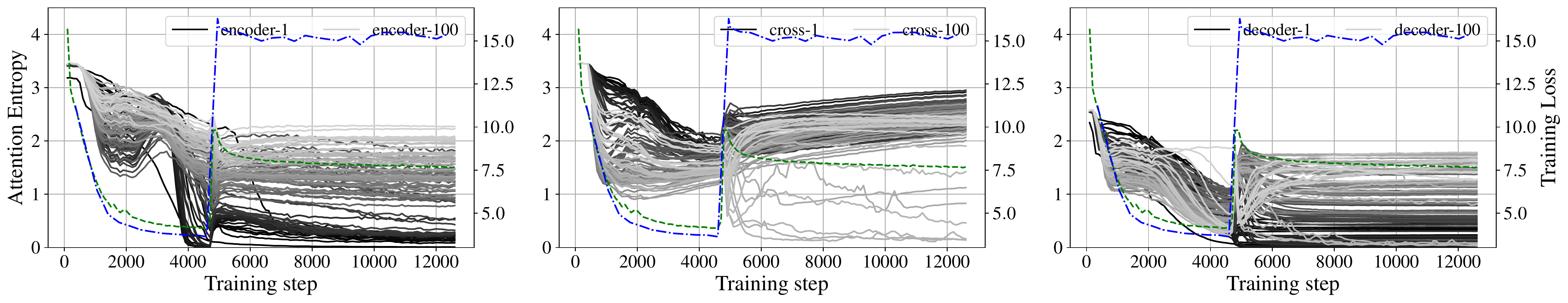}
    \includegraphics[scale=0.33]{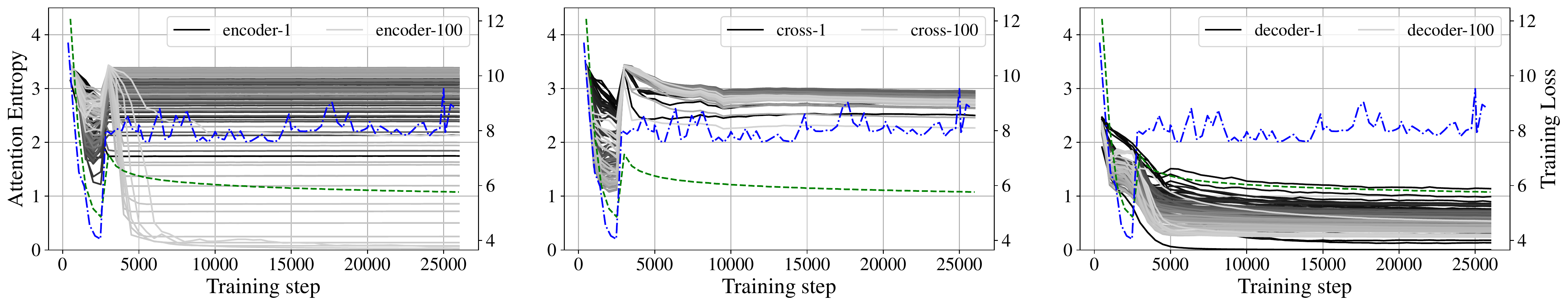}
    \includegraphics[scale=0.33]{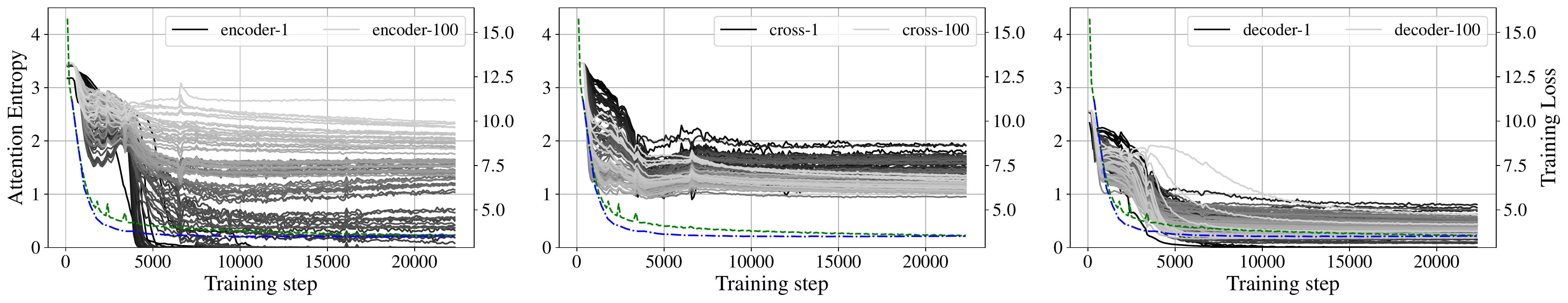}
    \caption{Attention entropy behaviour for MT models trained on WMT'17 with DeepNorm~\citep{wang2022deepnet} (residual rescaling and post-LN) for 50L-50L (row 1) and 100L-100L with three seeds (2-4 rows): encoder self-attention (left), encoder-decoder cross-attention (middle) and decoder self-attention (right). We plot training (dashed, green) and validation (dot-dashed, blue) losses and attention entropy across all Transformer layers (solid): a lighter color corresponds to a deeper layer. While DeepNorm solves vanishing gradient problem for deep models we observe attention entropy collapse phenomenon in both 50L-50L and 100L-100L models. While the 50L-50L model can recover from attention entropy collapse (happens in encoder layers) and nicely converge, 100L-100L suffers from it and can diverge. While~\citet{wang2022deepnet} reported stable training for 100L-100L we are unable to reproduce their results and observe 2/3 runs with different seeds (the rest of \hparams are the same as reported in the paper) diverge with attention entropy collapse.}
    \label{fig:mt:deepnorm}
\end{figure}

\begin{figure}[t!]
    \centering
    \includegraphics[scale=0.33]{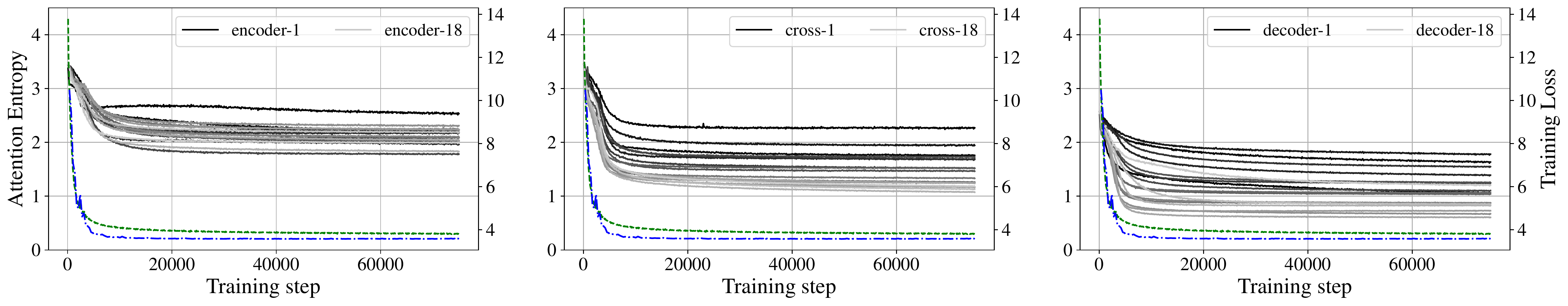}
    \includegraphics[scale=0.33]{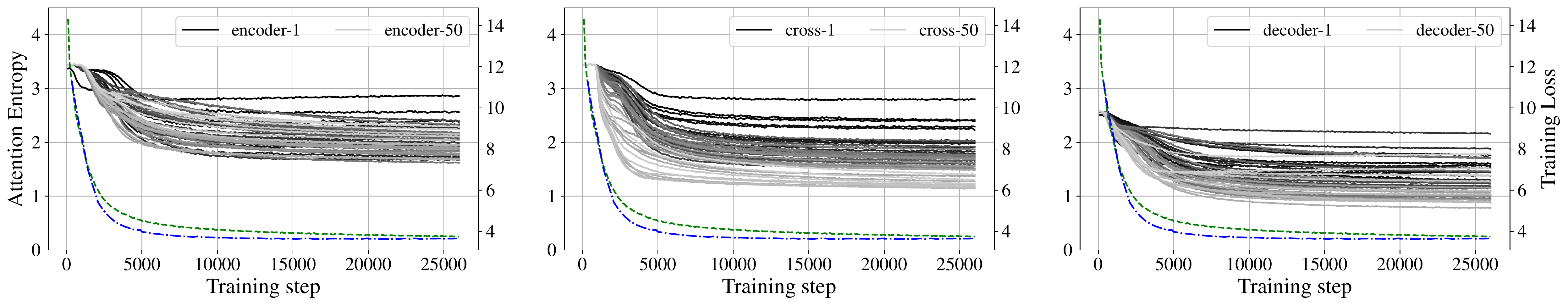}
    \caption{Attention entropy behaviour for MT models trained on WMT'17 with post-LN and \sigmareparam together for 18L-18L (top) and 50L-50L (bottom): encoder self-attention (left), encoder-decoder cross-attention (middle) and decoder self-attention (right). We plot training (dashed, green) and validation (dot-dashed, blue) losses and attention entropy across all Transformer layers (solid): a lighter color corresponds to a deeper layer. While 18L-18L is unstable for post-LN models (see top and middle in~\Cref{fig:mt:postln}), adding \sigmareparam nicely bounds attention entropy and stabilize training across different seeds and \hparams (we did not observe any instability or model divergence for $>10$ runs) allowing training with larger learning rates. While 50L-50L experiences vanishing gradients problem for post-LN models (see bottom in~\Cref{fig:mt:postln}), adding \sigmareparam balances gradients across layers and nicely bounds attention entropy: training is stable across different seeds and \hparams accepting larger learning rates.}
    \label{fig:mt:postln-spectral}
\end{figure}

\begin{figure}[t!]
    \centering
    \includegraphics[scale=0.33]{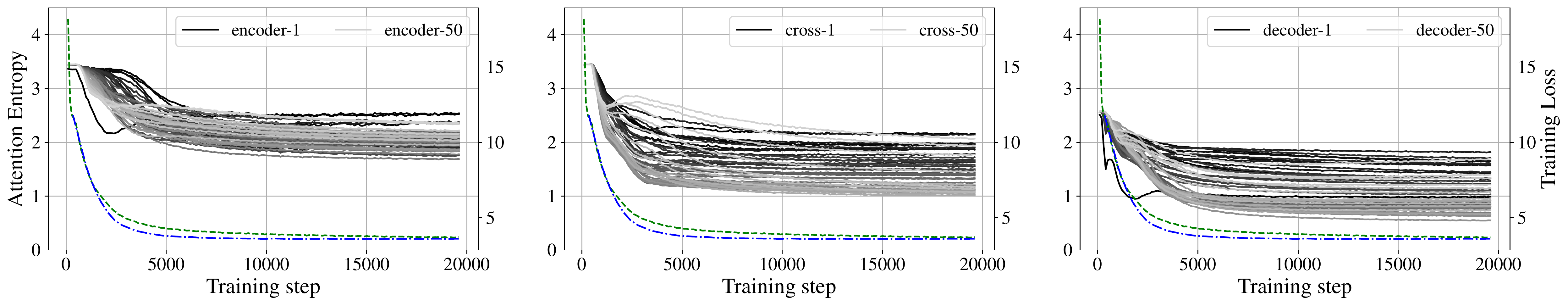}
    \includegraphics[scale=0.33]{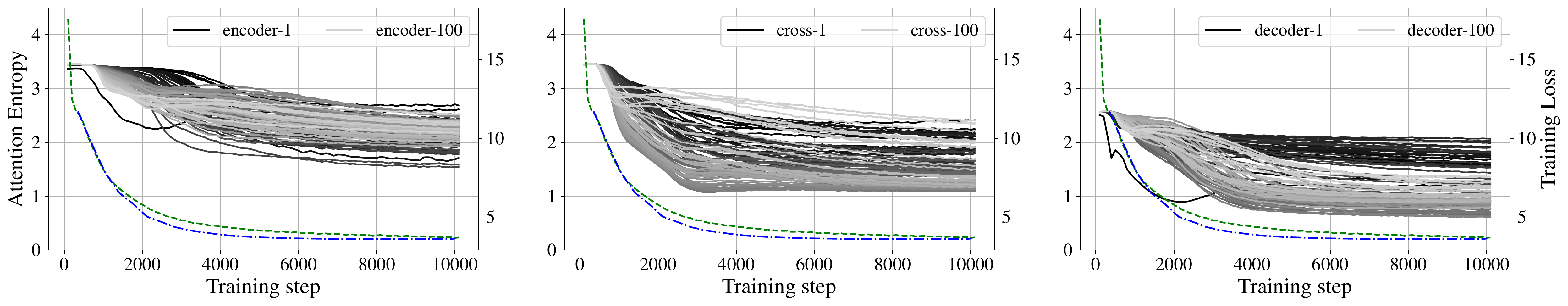}
    \caption{Attention entropy behaviour for MT models trained on WMT'17 with DeepNorm~\citep{wang2022deepnet} (residual rescaling and post-LN) and \sigmareparam together for 50L-50L (top) and 100L-100L (bottom): encoder self-attention (left), encoder-decoder cross-attention (middle) and decoder self-attention (right). We plot training (dashed, green) and validation (dot-dashed, blue) losses and attention entropy across all Transformer layers (solid): a lighter color corresponds to a deeper layer. Usage of \sigmareparam allows to nicely bound attention entropy (across encoder, decoder and cross-attention) for both 50L-50L and 100L-100L and fully stabilize training (across $>10$ runs with different \hparams and seeds we did not observe any instability and model divergence) of DeepNorm alone for 100L-100L allowing even larger learning rate values.}
    \label{fig:mt:deepnorm-spectral}
\end{figure}

All models performance on validation and test sets across depths as well as the number of successful runs are reported in \Cref{tab:mt}.

\subsection{\texorpdfstring{\sigmareparam}{Lg} for Deep Models}

We now experiment with injection of \sigmareparam into post-LN and DeepNorm models to alleviate attention entropy collapse and confirm \sigmareparam effectiveness for deep models. \sigmareparam is used for every linear layer in the encoder and decoder Transformer blocks alongside with post-LN. With DeepNorm we also apply its rescaling of initialization and residual connections.

\sigmareparam nicely bounds attention entropy for 18L-18L and 50L-50L post-LN models, resolving any divergence issues as well as vanishing gradient in the 50L-50L model, see \Cref{fig:mt:postln-spectral}. However, 100L-100L is still experiencing a vanishing gradient problem and only careful initialization of std for \sigmareparam can resolve it: for that reason we report that model training is not stable. In terms of performance, \sigmareparam with post-LN matches post-LN for 6L-6L, in the same ballpark for 18L-18L and performs the same as 18L-18L for 50L-50L. Note, that we did not do any \hparams search except tuning learning rate as \sigmareparam has different learning rate scales.

\sigmareparam also nicely bounds attention entropy for 18L-18L, 50L-50L, 100L-100L DeepNorm models, resolving any divergence issues for 100L-100L (vanishing gradient is not observed as DeepNorm targets it), see \Cref{fig:mt:postln-spectral}. In terms of performance \sigmareparam with DeepNorm matches DeepNorm for 6L-6L, in the same ballpark as DeepNorm for 18L-18L and inferior to DeepNorm for 50L-50L and 100L-100L.

\clearpage
\section{Language Modeling (LM)}\label{sec:language-modeling}
As we discussed above encoder Transformer for vision and speech domains and encoder-decoder for machine translation, in this section we focus on the pure decoder architecture in language model task to verify if \sigmareparam is effective for stable training and can simplify a training recipe there too.

\subsection{Experimental Outline}
We use the WikiText-103 language model (LM) benchmark, which consists of 103M tokens sampled from English Wikipedia \citep{wiki103}. Our baseline is a highly optimized Transformer \citep{Baevski2019AdaptiveIR} with 32 layers, 8 heads, 128 head dimensions, 1024 model dimensions, 4096 fully connected dimensions and post-LayerNorm (post-LN). The word embedding and softmax matrices are tied \citep{Press2017UsingTO}.
We partition the training data into non-overlapping blocks of 512 contiguous tokens and train the model to autoregressively predict each token \citep{Baevski2019AdaptiveIR}.
Validation and test perplexity is measured by predicting the last 256 words out of the input of 512 consecutive words to avoid evaluating tokens in the beginning with limited context (\textit{early token curse}, \citealp{shortformer}). We integrate \sigmareparam implementation into the open-sourced code and recipe for the baseline\footnote{\url{https://github.com/facebookresearch/fairseq/blob/main/examples/language_model/README.adaptive_inputs.md}}. All models are trained in full precision on 8 GPUs of A100 40GB.

\subsection{Results}
We do not experience training instability with the baseline Transformer, likely because the masked attention in autoregressive models makes entropy collapse less likely to occur. This is consistent and in line with observations in machine translation where entropy collapse is observed in the encoder and cross-attention. Nonetheless, we experimented with \sigmareparam to test its generality on a different modality/problem. We apply \sigmareparam to all linear layers of the Transformer while removing all post-LNs, and search for learning rate in a grid [1, 1.5, 2, 2.5] and weight decay in the grid [1e-3, 1e-4, 0]. All other \hparams are kept the same as the baseline, including Nesterov SGD optimizer\footnote{Note, this is different from other domains where a standard recipe includes only adaptive optimizers.}. The results are shown in Table \ref{tab:lm_results}. We see that even in the absence of LayerNorm, \sigmareparam shows strong performance in convergence and validation/test performance. With a mild weight decay, \sigmareparam also outperforms the baseline wrt the validation/test PPL. In summary, {\it while there is no observed entropy collapse in language model training, \sigmareparam can simplify a training recipe by removing all post-LNs}.

\begin{table}[ht!]
\caption{
WikiText-103 language modeling results in perplexity.}
\label{tab:lm_results}
\centering
\addtolength{\tabcolsep}{-2.7pt}  
\begin{tabular}{@{} lcccccc @{}}
\toprule

\multirow{2}{*}{Model} &  \multicolumn{3}{c}{PPL$\downarrow$} \\
\cmidrule(lr){2-4}
& train & valid  & test \\
\midrule
$\sigma$Reparam w/ weight decay  & 16.5 & \textbf{17.9} & \textbf{18.6} \\
$\sigma$Reparam w/o weight decay  & \textbf{12.9} & 18.5 & 19.3 \\
\midrule
post-LN \citet{Baevski2019AdaptiveIR}  & 15.4 & 18.1 & 18.7 \\
\bottomrule
\end{tabular}
\end{table}

\clearpage
\section{\Hparams for Supervised Vision}\label{appendix:hps}
\label{appendix:image_cls}
As mentioned in Section \ref{sec:supimg} we compare \sigmareparam against DeiT \citep{touvron2021training} and MAE \citep{he2022masked} supervised training recipes for vision Transformers. In Table \ref{tab:sup-variants} we highlight the differences between DeiT, MAE supervised and \sigmareparam. \sigmareparam presents a simplified and stable training objective for ViT-B variants. In Table \ref{tab:sup-variants-l} we present the same comparing the ViT-L variants. There is no exact 1:1 comparison for a ViT-L with the DeiT training framework so we only compare against the MAE supervised model. 
\begin{table}[h]
  \caption{Training \hparams comparison for supervised ViT-B/16.}
  \label{tab:sup-variants}
  \centering
  \small
  \begin{tabular}{lcccc}
    \toprule
    & DeiT & MAE & \sigmareparam\\
    \midrule
    Top-1 & 81.8\% & 82.1\% & 81.88\% \\
    EMA Top-1 & - & 82.3\% & 82.37\% \\
    \midrule
    Weight Init & \texttt{trunc\_normal(.02)} & \texttt{trunc\_normal(.02)} & \texttt{trunc\_normal(.02)} \\
    Patcher Init     & \texttt{trunc\_normal(.02)} & \texttt{trunc\_normal(.02)} & \texttt{trunc\_normal(.02)} \\
    \sigmareparam     & No & No & Yes \\
    Layer Norm     & Yes & Yes & \textbf{No} \\
    Optimizer & AdamW($\beta_1$=0.9, $\beta_2$=0.95) & AdamW($\beta_1$=0.9, $\beta_2$=0.95) & \textbf{LARS(mom=0.9)} \\
    Base LR & $5\times 10^{-4}$ & $1\times 10^{-4}$ & 0.1 \\
    LR schedule & cosine & cosine & \textbf{step(210, 0.1)} \\
    Batch size & 1024 & 4096 & 4096 \\
    Weight decay & 0.05 & 0.3 & \textbf{0.0} \\
    Warmup epochs & 5 & 20 & \textbf{0} \\
    Training epochs & 300 & 300 & \textbf{250} \\
    Label smoothing & 0.1 & 0.1 & 0.1 \\
    Stoch. Depth & 0.1 & 0.1 & 0.1 \\
    Repeated Aug. & 2 & 2 & 2 \\
    RandAug & 9/0.5 & 9/0.5 & 9/0.5 \\
    Mixup prob. & 0.8 & 0.8 & 0.8 \\
    Cutmix prob. & 1.0 & 1.0 & 1.0 \\
    Erasing prob. & 0.25 & 0.25 & 0.25 \\
    \bottomrule
  \end{tabular}
\end{table}

\begin{table}[ht!]
  \caption{Training \hparams comparison for supervised ViT-L/16.}
  \label{tab:sup-variants-l}
  \centering
  \small
  \begin{tabular}{lcccc}
    \toprule
    & MAE & \sigmareparam\\
    \midrule
    Top-1 & 81.5\% & 82.41\% \\
    EMA Top-1 & 82.6\% & 82.48\% \\
    \midrule
    Weight Init & \texttt{trunc\_normal(.02)} & \texttt{trunc\_normal(.01)} \\
    Patcher Init     &  \texttt{trunc\_normal(.02)} & \texttt{trunc\_normal(.0025)} \\
    \sigmareparam     & No & Yes \\
    Layer Norm     & Yes & \textbf{No} \\
    Optimizer & AdamW($\beta_1$=0.9, $\beta_2$=0.95) & \textbf{LARS(mom=0.9)} \\
    Base LR & $1\times 10^{-4}$ & 0.15 \\
    LR schedule & cosine & cosine \\
    Batch size & 4096 & 4096 \\
    Weight decay & 0.3 & \textbf{0.0} \\
    Warmup epochs & 20 & \textbf{0} \\
    Training epochs & 300 & 300 \\
    Label smoothing & 0.1 & 0.1 \\
    Stoch. Depth & 0.2 & 0.2 \\
    Repeated Aug. & 2 & 2 \\
    RandAug & 9/0.5 & 9/0.5 \\
    Mixup prob. & 0.8 & 0.8 \\
    Cutmix prob. & 1.0 & 1.0 \\
    Erasing prob. & 0.25 & 0.25 \\
    \bottomrule
  \end{tabular}
\end{table}

\clearpage

\section{Ablations}

\paragraph{Initialization for \texorpdfstring{$\sigma$}{}Reparam}

First, we found that it is better to initialize $\gamma$ as 1 and not compute it from the initialized kernel as there could be different values for spectral norm depending on the initialization of the kernel. 
In this case we observed values greater than 1 for the spectral norm which cause divergence / no training. 
We compared different initializations for the kernel and we did not see any differences in initialization (e.g. uniform, normal).
The only factor that influences training behavior is the standard deviation (std) of the initialization pdf, which also influences effective learning rate. 
In speech recognition we found that training is robust with respect to changes of std (Figure~\ref{fig:asr_robustness}), however larger std performs better and sweet spot is 0.2-0.3. In machine translation models are also robust to the choice of std, however some architectures perform better with std of 0.01 while others with 0.1 std. In language modeling we observed robust performance with respect to initialization, and we use the default initialization from the Transformer baseline for all experiments.

In vision we initialize the \sigmareparam $\gamma$ term using the first singular value, computed with the SVD at weight initialization. We then use one power iteration for all further updates. We provide weight and patcher initializations for the ViT-B/16 in Table \ref{tab:sup-variants} and the ViT-L/16 and ViT-H/14 in Table \ref{tab:sup-variants-l}.

\paragraph{Separate \texorpdfstring{$\sigma$}{}Reparam for key, queries and values}
We found that they behave more or less similar while separate normalization allows to achieve lower training loss due to larger capacity ability which provides potential to scale. However, for ASR training with LARS it is better to have joint reparametrization to achieve stable training and comparable results with adaptive optimizers, see Section~\ref{sec:asr_lars}.

\section{Discussion}
We believe that our experiments have covered representative domains, architectures and training losses for typical Transformer applications. The key factors that affect training stability are the initial token embedding layer (where for vision and speech tasks they are based on convolution projections, and for MT and language modeling are with word embeddings), topology of architecture (encoder mode for vision and speech, encoder-decoder for MT and decoder for language modeling), and the training loss (cross-entropy loss, contrastive and CTC loss). While each of these design choices may bring their own challenges for training, we show that entropy collapse is a common symptom accompanying instability and \sigmareparam is a general fix compatible with all settings.

\section{Contributions}
All authors contributed into writing the manuscript, designing experiments and discussion of all results at every stage of the project.

\paragraph{Attention Entropy Collapse Observations}
All initial experiments and findings of the attention entropy collapse phenomenon for ViT models on ImageNet are done by Shuangfei Zhai. Preliminary theoretical analysis and proposal to use \sigmareparam as a solution is also done by Shuangfei Zhai.

\paragraph{Theory} All theoretical results, Appendix~\ref{sec:proof}, are done by Etai Littwin. Review of proofs is done by Tatiana Likhomanenko.

\paragraph{Causality Analysis} Simulated case-control study investigation and all related experimental work done by Dan Busbridge.
Tatiana Likhomanenko, Etai Littwin, Jason Ramapuram, Russ Webb and Shuangfei Zhai helped with designing the experimental setting for intervention methodology.

\paragraph{Supervised Learning in Vision} Shuangfei Zhai conducted the initial \sigmareparam vision experiments with DeiT and made the initial observations of relaxing / removing weight decay from \sigmareparam. Jason Ramapuram scaled and conducted the remaining supervised vision experiments and analysis (including the MAE, weight-norm and spectral-norm baselines) over ImageNet1k (Table \ref{tab:vision_summary}) and Imagenet21k (Table \ref{tab:imagenet21k}) and enabled the drastically simplified \sigmareparam vision transformer recipe from Table \ref{tab:vision_summary}. This simplified \sigmareparam recipe enables SGD (LARS) training over fewer epochs and completely removes \{weight-decay, pre-LN layers, LR warmup and the LR cosine schedule\}.

\paragraph{Self-Supervised Learning in Vision}
Known issues with SimCLR~\citep{chen2020simple} stability that were observed in~\citet{chen2021empirical} pointed out by Vimal Thilak.
All investigations, experiments and related analysis done by Dan Busbridge.

\paragraph{Automatic Speech Recognition}
All speech recognition experiments are done by Tatiana Likhomanenko. Shuangfei Zhai and Jason Ramapuram advised to have also large scale results.

\paragraph{Machine Translation}
Initial implementation and experiments of \sigmareparam applicability to MT were done on WMT'14 by Jiatao Gu. 
Later, Tatiana Likhomanenko pushed to investigate deep transformer models and their stability.
Jason Ramapuram pointed to the deepnorm~\citep{wang2022deepnet} results to probe for entropy collapse phenomenon.
All later experiments, Section~\ref{sec:mt} and Appendix~\ref{sec:app:mt}, with deep transformers and deepnorm are done by Tatiana Likhomanenko.

\paragraph{Language Modeling}
Initial implementation and preliminary results on applicability of \sigmareparam to the language modeling, Appendix~\ref{sec:language-modeling}, are done by Yizhe Zhang with help from Jiatao Gu. Shuangfei Zhai contributed to the experiments and obtained the final results.

\paragraph{Implementation Details and Ablations}
Investigation into how initialization influences \sigmareparam is done in parallel in different domains and experiments by Jason Ramapuram, Shuangfei Zhai and Tatiana Likhomanenko. Investigation of different variants (with stop gradient, with different matrices) of \sigmareparam was done by Shuangfei Zhai, Tatiana Likhomanenko and Jason Ramapuram. Investigation of full precision training vs mixed precision training was done by Tatiana Likhomanenko, Dan Busbridge and Jason Ramapuram. 

Implementation is done in 2 frameworks (PyTorch and Jax) and in 5 codebases. The initial implementation of \sigmareparam module is done in PyTorch by Shuangfei Zhai, with further reimplementation in Jax by Tatiana Likhomanenko. Later the implementation was integrated and adopted into other baseline toolboxes by Jason Ramapuram, Dan Busbridge, Yizhe Zhang, Tatiana Likhomanenko and Jiatao Gu.


\begin{thebibliography}{55}
\providecommand{\natexlab}[1]{#1}
\providecommand{\url}[1]{\texttt{#1}}
\expandafter\ifx\csname urlstyle\endcsname\relax
  \providecommand{\doi}[1]{doi: #1}\else
  \providecommand{\doi}{doi: \begingroup \urlstyle{rm}\Url}\fi

\bibitem[Assran et~al.(2022)Assran, Caron, Misra, Bojanowski, Bordes, Vincent,
  Joulin, Rabbat, and Ballas]{DBLP:journals/corr/abs-2204-07141}
Mahmoud Assran, Mathilde Caron, Ishan Misra, Piotr Bojanowski, Florian Bordes,
  Pascal Vincent, Armand Joulin, Michael~G. Rabbat, and Nicolas Ballas.
\newblock Masked siamese networks for label-efficient learning.
\newblock \emph{CoRR}, abs/2204.07141, 2022.
\newblock \doi{10.48550/arXiv.2204.07141}.
\newblock URL \url{https://doi.org/10.48550/arXiv.2204.07141}.

\bibitem[Ba et~al.(2016)Ba, Kiros, and Hinton]{ba2016layer}
Jimmy~Lei Ba, Jamie~Ryan Kiros, and Geoffrey~E Hinton.
\newblock Layer normalization.
\newblock \emph{arXiv preprint arXiv:1607.06450}, 2016.

\bibitem[Bachlechner et~al.(2021)Bachlechner, Majumder, Mao, Cottrell, and
  McAuley]{bachlechner2021rezero}
Thomas Bachlechner, Bodhisattwa~Prasad Majumder, Henry Mao, Gary Cottrell, and
  Julian McAuley.
\newblock Rezero is all you need: Fast convergence at large depth.
\newblock In \emph{Uncertainty in Artificial Intelligence}, pp.\  1352--1361.
  PMLR, 2021.

\bibitem[Baevski \& Auli(2019)Baevski and Auli]{Baevski2019AdaptiveIR}
Alexei Baevski and Michael Auli.
\newblock Adaptive input representations for neural language modeling.
\newblock In \emph{Proc.\ of ICLR}, 2019.
\newblock URL \url{https://arxiv.org/abs/1809.10853}.

\bibitem[Caron et~al.(2021)Caron, Touvron, Misra, J{\'{e}}gou, Mairal,
  Bojanowski, and Joulin]{DBLP:conf/iccv/CaronTMJMBJ21}
Mathilde Caron, Hugo Touvron, Ishan Misra, Herv{\'{e}} J{\'{e}}gou, Julien
  Mairal, Piotr Bojanowski, and Armand Joulin.
\newblock Emerging properties in self-supervised vision transformers.
\newblock In \emph{2021 {IEEE/CVF} International Conference on Computer Vision,
  {ICCV} 2021, Montreal, QC, Canada, October 10-17, 2021}, pp.\  9630--9640.
  {IEEE}, 2021.
\newblock \doi{10.1109/ICCV48922.2021.00951}.
\newblock URL \url{https://doi.org/10.1109/ICCV48922.2021.00951}.

\bibitem[Chen et~al.(2020)Chen, Kornblith, Norouzi, and
  Hinton]{DBLP:conf/icml/ChenK0H20}
Ting Chen, Simon Kornblith, Mohammad Norouzi, and Geoffrey~E. Hinton.
\newblock A simple framework for contrastive learning of visual
  representations.
\newblock In \emph{Proceedings of the 37th International Conference on Machine
  Learning, {ICML} 2020, 13-18 July 2020, Virtual Event}, volume 119 of
  \emph{Proceedings of Machine Learning Research}, pp.\  1597--1607. {PMLR},
  2020.
\newblock URL \url{http://proceedings.mlr.press/v119/chen20j.html}.

\bibitem[Chen et~al.(2021{\natexlab{a}})Chen, Hsieh, and Gong]{Chen2021WhenVT}
Xiangning Chen, Cho-Jui Hsieh, and Boqing Gong.
\newblock When vision transformers outperform resnets without pretraining or
  strong data augmentations.
\newblock \emph{ArXiv}, abs/2106.01548, 2021{\natexlab{a}}.

\bibitem[Chen et~al.(2021{\natexlab{b}})Chen, Xie, and He]{chen2021empirical}
Xinlei Chen, Saining Xie, and Kaiming He.
\newblock An empirical study of training self-supervised vision transformers.
\newblock In \emph{Proceedings of the IEEE/CVF International Conference on
  Computer Vision}, pp.\  9640--9649, 2021{\natexlab{b}}.

\bibitem[Cohen et~al.(2021)Cohen, Kaur, Li, Kolter, and
  Talwalkar]{DBLP:conf/iclr/CohenKLKT21}
Jeremy~M. Cohen, Simran Kaur, Yuanzhi Li, J.~Zico Kolter, and Ameet Talwalkar.
\newblock Gradient descent on neural networks typically occurs at the edge of
  stability.
\newblock In \emph{9th International Conference on Learning Representations,
  {ICLR} 2021, Virtual Event, Austria, May 3-7, 2021}. OpenReview.net, 2021.
\newblock URL \url{https://openreview.net/forum?id=jh-rTtvkGeM}.

\bibitem[Cohen et~al.(2022)Cohen, Ghorbani, Krishnan, Agarwal, Medapati,
  Badura, Suo, Cardoze, Nado, Dahl, and
  Gilmer]{DBLP:journals/corr/abs-2207-14484}
Jeremy~M. Cohen, Behrooz Ghorbani, Shankar Krishnan, Naman Agarwal, Sourabh
  Medapati, Michal Badura, Daniel Suo, David Cardoze, Zachary Nado, George~E.
  Dahl, and Justin Gilmer.
\newblock Adaptive gradient methods at the edge of stability.
\newblock \emph{CoRR}, abs/2207.14484, 2022.
\newblock \doi{10.48550/arXiv.2207.14484}.
\newblock URL \url{https://doi.org/10.48550/arXiv.2207.14484}.

\bibitem[Deng et~al.(2009)Deng, Dong, Socher, Li, Li, and
  Fei-Fei]{deng2009imagenet}
Jia Deng, Wei Dong, Richard Socher, Li-Jia Li, Kai Li, and Li~Fei-Fei.
\newblock Imagenet: A large-scale hierarchical image database.
\newblock In \emph{2009 IEEE conference on computer vision and pattern
  recognition}, pp.\  248--255. Ieee, 2009.

\bibitem[Ding et~al.(2021)Ding, Zhang, Ma, Han, Ding, and Sun]{ding2021repvgg}
Xiaohan Ding, Xiangyu Zhang, Ningning Ma, Jungong Han, Guiguang Ding, and Jian
  Sun.
\newblock Repvgg: Making vgg-style convnets great again.
\newblock In \emph{Proceedings of the IEEE/CVF Conference on Computer Vision
  and Pattern Recognition}, pp.\  13733--13742, 2021.

\bibitem[Dong et~al.(2022)Dong, Bao, Zhang, Chen, Gu, Zhang, Yuan, Chen, Wen,
  and Yu]{DBLP:journals/corr/abs-2212-06138}
Xiaoyi Dong, Jianmin Bao, Ting Zhang, Dongdong Chen, Shuyang Gu, Weiming Zhang,
  Lu~Yuan, Dong Chen, Fang Wen, and Nenghai Yu.
\newblock {CLIP} itself is a strong fine-tuner: Achieving 85.7{\%} and 88.0{\%}
  top-1 accuracy with vit-b and vit-l on imagenet.
\newblock \emph{CoRR}, abs/2212.06138, 2022.
\newblock \doi{10.48550/arXiv.2212.06138}.
\newblock URL \url{https://doi.org/10.48550/arXiv.2212.06138}.

\bibitem[Dong et~al.(2021)Dong, Cordonnier, and Loukas]{dong2021attention}
Yihe Dong, Jean-Baptiste Cordonnier, and Andreas Loukas.
\newblock Attention is not all you need: Pure attention loses rank doubly
  exponentially with depth.
\newblock In \emph{International Conference on Machine Learning}, pp.\
  2793--2803. PMLR, 2021.

\bibitem[Dosovitskiy et~al.(2020)Dosovitskiy, Beyer, Kolesnikov, Weissenborn,
  Zhai, Unterthiner, Dehghani, Minderer, Heigold, Gelly,
  et~al.]{dosovitskiy2020image}
Alexey Dosovitskiy, Lucas Beyer, Alexander Kolesnikov, Dirk Weissenborn,
  Xiaohua Zhai, Thomas Unterthiner, Mostafa Dehghani, Matthias Minderer, Georg
  Heigold, Sylvain Gelly, et~al.
\newblock An image is worth 16x16 words: Transformers for image recognition at
  scale.
\newblock \emph{arXiv preprint arXiv:2010.11929}, 2020.

\bibitem[Duchi et~al.(2011)Duchi, Hazan, and Singer]{duchi2011adaptive}
John Duchi, Elad Hazan, and Yoram Singer.
\newblock Adaptive subgradient methods for online learning and stochastic
  optimization.
\newblock \emph{Journal of machine learning research}, 12\penalty0
  (Jul):\penalty0 2121--2159, 2011.

\bibitem[Foret et~al.(2020)Foret, Kleiner, Mobahi, and
  Neyshabur]{Foret2020SharpnessAwareMF}
Pierre Foret, Ariel Kleiner, Hossein Mobahi, and Behnam Neyshabur.
\newblock Sharpness-aware minimization for efficiently improving
  generalization.
\newblock \emph{ArXiv}, abs/2010.01412, 2020.

\bibitem[Ghorbani et~al.(2019)Ghorbani, Krishnan, and
  Xiao]{DBLP:conf/icml/GhorbaniKX19}
Behrooz Ghorbani, Shankar Krishnan, and Ying Xiao.
\newblock An investigation into neural net optimization via hessian eigenvalue
  density.
\newblock In Kamalika Chaudhuri and Ruslan Salakhutdinov (eds.),
  \emph{Proceedings of the 36th International Conference on Machine Learning,
  {ICML} 2019, 9-15 June 2019, Long Beach, California, {USA}}, volume~97 of
  \emph{Proceedings of Machine Learning Research}, pp.\  2232--2241. {PMLR},
  2019.
\newblock URL \url{http://proceedings.mlr.press/v97/ghorbani19b.html}.

\bibitem[Gilmer et~al.(2021)Gilmer, Ghorbani, Garg, Kudugunta, Neyshabur,
  Cardoze, Dahl, Nado, and Firat]{DBLP:journals/corr/abs-2110-04369}
Justin Gilmer, Behrooz Ghorbani, Ankush Garg, Sneha Kudugunta, Behnam
  Neyshabur, David Cardoze, George~E. Dahl, Zachary Nado, and Orhan Firat.
\newblock A loss curvature perspective on training instability in deep
  learning.
\newblock \emph{CoRR}, abs/2110.04369, 2021.
\newblock URL \url{https://arxiv.org/abs/2110.04369}.

\bibitem[Graves et~al.(2006)Graves, Fern{\'a}ndez, Gomez, and
  Schmidhuber]{graves2006connectionist}
Alex Graves, Santiago Fern{\'a}ndez, Faustino Gomez, and J{\"u}rgen
  Schmidhuber.
\newblock Connectionist temporal classification: labelling unsegmented sequence
  data with recurrent neural networks.
\newblock In \emph{Proceedings of the 23rd international conference on Machine
  learning}, pp.\  369--376, 2006.

\bibitem[He et~al.(2022)He, Chen, Xie, Li, Doll{\'a}r, and
  Girshick]{he2022masked}
Kaiming He, Xinlei Chen, Saining Xie, Yanghao Li, Piotr Doll{\'a}r, and Ross
  Girshick.
\newblock Masked autoencoders are scalable vision learners.
\newblock In \emph{Proceedings of the IEEE/CVF Conference on Computer Vision
  and Pattern Recognition}, pp.\  16000--16009, 2022.

\bibitem[Huang et~al.(2020)Huang, Perez, Ba, and Volkovs]{huang2020improving}
Xiao~Shi Huang, Felipe Perez, Jimmy Ba, and Maksims Volkovs.
\newblock Improving transformer optimization through better initialization.
\newblock In \emph{International Conference on Machine Learning}, pp.\
  4475--4483. PMLR, 2020.

\bibitem[Kingma \& Ba(2015)Kingma and Ba]{DBLP:journals/corr/KingmaB14}
Diederik~P. Kingma and Jimmy Ba.
\newblock Adam: {A} method for stochastic optimization.
\newblock In Yoshua Bengio and Yann LeCun (eds.), \emph{3rd International
  Conference on Learning Representations, {ICLR} 2015, San Diego, CA, USA, May
  7-9, 2015, Conference Track Proceedings}, 2015.
\newblock URL \url{http://arxiv.org/abs/1412.6980}.

\bibitem[Li et~al.(2022)Li, Bhojanapalli, Zaheer, Reddi, and
  Kumar]{li2022robust}
Zhiyuan Li, Srinadh Bhojanapalli, Manzil Zaheer, Sashank Reddi, and Sanjiv
  Kumar.
\newblock Robust training of neural networks using scale invariant
  architectures.
\newblock In \emph{International Conference on Machine Learning}, pp.\
  12656--12684. PMLR, 2022.

\bibitem[Likhomanenko et~al.(2021{\natexlab{a}})Likhomanenko, Xu, Kahn,
  Synnaeve, and Collobert]{likhomanenko2020slimipl}
Tatiana Likhomanenko, Qiantong Xu, Jacob Kahn, Gabriel Synnaeve, and Ronan
  Collobert.
\newblock slimipl: Language-model-free iterative pseudo-labeling.
\newblock \emph{Proc. Interspeech}, 2021{\natexlab{a}}.

\bibitem[Likhomanenko et~al.(2021{\natexlab{b}})Likhomanenko, Xu, Pratap,
  Tomasello, Kahn, Avidov, Collobert, and Synnaeve]{likhomanenko2020rethinking}
Tatiana Likhomanenko, Qiantong Xu, Vineel Pratap, Paden Tomasello, Jacob Kahn,
  Gilad Avidov, Ronan Collobert, and Gabriel Synnaeve.
\newblock Rethinking evaluation in asr: Are our models robust enough?
\newblock \emph{Proc. Interspeech}, 2021{\natexlab{b}}.

\bibitem[Likhomanenko et~al.(2021{\natexlab{c}})Likhomanenko, Xu, Synnaeve,
  Collobert, and Rogozhnikov]{likhomanenko2021cape}
Tatiana Likhomanenko, Qiantong Xu, Gabriel Synnaeve, Ronan Collobert, and Alex
  Rogozhnikov.
\newblock Cape: Encoding relative positions with continuous augmented
  positional embeddings.
\newblock \emph{Advances in Neural Information Processing Systems}, 34,
  2021{\natexlab{c}}.

\bibitem[Liu et~al.(2020{\natexlab{a}})Liu, Liu, Gao, Chen, and
  Han]{liu2020admin}
Liyuan Liu, Xiaodong Liu, Jianfeng Gao, Weizhu Chen, and Jiawei Han.
\newblock Understanding the difficulty of training transformers.
\newblock In \emph{Proceedings of the 2020 Conference on Empirical Methods in
  Natural Language Processing (EMNLP 2020)}, 2020{\natexlab{a}}.

\bibitem[Liu et~al.(2020{\natexlab{b}})Liu, Duh, Liu, and Gao]{liu_deep_2020}
Xiaodong Liu, Kevin Duh, Liyuan Liu, and Jianfeng Gao.
\newblock Very deep transformers for neural machine translation.
\newblock In \emph{arXiv:2008.07772 [cs]}, 2020{\natexlab{b}}.

\bibitem[Merity et~al.(2017)Merity, Xiong, Bradbury, and Socher]{wiki103}
Stephen Merity, Caiming Xiong, James Bradbury, and Richard Socher.
\newblock Pointer sentinel mixture models.
\newblock In \emph{Proc.\ of ICLR}, 2017.
\newblock URL \url{https://arxiv.org/abs/1609.07843}.

\bibitem[Mises \& Pollaczek-Geiringer(1929)Mises and
  Pollaczek-Geiringer]{mises1929praktische}
RV~Mises and Hilda Pollaczek-Geiringer.
\newblock Praktische verfahren der gleichungsaufl{\"o}sung.
\newblock \emph{ZAMM-Journal of Applied Mathematics and Mechanics/Zeitschrift
  f{\"u}r Angewandte Mathematik und Mechanik}, 9\penalty0 (1):\penalty0 58--77,
  1929.

\bibitem[Miyato et~al.(2018)Miyato, Kataoka, Koyama, and
  Yoshida]{miyato2018spectral}
Takeru Miyato, Toshiki Kataoka, Masanori Koyama, and Yuichi Yoshida.
\newblock Spectral normalization for generative adversarial networks.
\newblock In \emph{International Conference on Learning Representations}, 2018.

\bibitem[Nguyen \& Salazar(2019)Nguyen and Salazar]{nguyen2019transformers}
Toan~Q Nguyen and Julian Salazar.
\newblock Transformers without tears: Improving the normalization of
  self-attention.
\newblock In \emph{Proceedings of the 16th International Conference on Spoken
  Language Translation}, 2019.

\bibitem[Noci et~al.(2022)Noci, Anagnostidis, Biggio, Orvieto, Singh, and
  Lucchi]{noci2022signal}
Lorenzo Noci, Sotiris Anagnostidis, Luca Biggio, Antonio Orvieto, Sidak~Pal
  Singh, and Aurelien Lucchi.
\newblock Signal propagation in transformers: Theoretical perspectives and the
  role of rank collapse.
\newblock \emph{arXiv preprint arXiv:2206.03126}, 2022.

\bibitem[Okuta et~al.(2017)Okuta, Unno, Nishino, Hido, and
  Loomis]{cupy_learningsys2017}
Ryosuke Okuta, Yuya Unno, Daisuke Nishino, Shohei Hido, and Crissman Loomis.
\newblock Cupy: A numpy-compatible library for nvidia gpu calculations.
\newblock In \emph{Proceedings of Workshop on Machine Learning Systems
  (LearningSys) in The Thirty-first Annual Conference on Neural Information
  Processing Systems (NIPS)}, 2017.
\newblock URL \url{http://learningsys.org/nips17/assets/papers/paper_16.pdf}.

\bibitem[Ott et~al.(2019)Ott, Edunov, Baevski, Fan, Gross, Ng, Grangier, and
  Auli]{ott2019fairseq}
Myle Ott, Sergey Edunov, Alexei Baevski, Angela Fan, Sam Gross, Nathan Ng,
  David Grangier, and Michael Auli.
\newblock fairseq: A fast, extensible toolkit for sequence modeling.
\newblock In \emph{Proceedings of the 2019 Conference of the North American
  Chapter of the Association for Computational Linguistics (Demonstrations)},
  pp.\  48--53, 2019.

\bibitem[Panayotov et~al.(2015)Panayotov, Chen, Povey, and
  Khudanpur]{panayotov2015librispeech}
Vassil Panayotov, Guoguo Chen, Daniel Povey, and Sanjeev Khudanpur.
\newblock Librispeech: an asr corpus based on public domain audio books.
\newblock In \emph{2015 IEEE International Conference on Acoustics, Speech and
  Signal Processing (ICASSP)}, pp.\  5206--5210. IEEE, 2015.

\bibitem[Park et~al.(2019)Park, Chan, Zhang, Chiu, Zoph, Cubuk, and
  Le]{park2019specaug}
Daniel~S Park, William Chan, Yu~Zhang, Chung-Cheng Chiu, Barret Zoph, Ekin~D
  Cubuk, and Quoc~V Le.
\newblock Specaugment: A simple data augmentation method for automatic speech
  recognition.
\newblock \emph{Proc. Interspeech 2019}, pp.\  2613--2617, 2019.

\bibitem[Pearl(2009)]{Pearl09}
Judea Pearl.
\newblock \emph{Causality}.
\newblock Cambridge University Press, Cambridge, UK, 2 edition, 2009.
\newblock ISBN 978-0-521-89560-6.
\newblock \doi{10.1017/CBO9780511803161}.

\bibitem[Press \& Wolf(2017)Press and Wolf]{Press2017UsingTO}
Ofir Press and Lior Wolf.
\newblock Using the output embedding to improve language models.
\newblock In \emph{Proc.\ of EACL}, 2017.
\newblock URL \url{https://arxiv.org/abs/1608.05859}.

\bibitem[Press et~al.(2021)Press, Smith, and Lewis]{shortformer}
Ofir Press, Noah~A. Smith, and Mike Lewis.
\newblock Shortformer: Better language modeling using shorter inputs, 2021.
\newblock URL \url{https://arxiv.org/abs/2012.15832}.

\bibitem[Radford et~al.(2019)Radford, Wu, Child, Luan, Amodei, Sutskever,
  et~al.]{radford2019language}
Alec Radford, Jeffrey Wu, Rewon Child, David Luan, Dario Amodei, Ilya
  Sutskever, et~al.
\newblock Language models are unsupervised multitask learners.
\newblock \emph{OpenAI blog}, 1\penalty0 (8):\penalty0 9, 2019.

\bibitem[Ridnik et~al.(2021)Ridnik, Ben-Baruch, Noy, and
  Zelnik-Manor]{ridnik2021imagenet21k}
Tal Ridnik, Emanuel Ben-Baruch, Asaf Noy, and Lihi Zelnik-Manor.
\newblock Imagenet-21k pretraining for the masses, 2021.

\bibitem[Salimans \& Kingma(2016)Salimans and Kingma]{salimans2016weight}
Tim Salimans and Durk~P Kingma.
\newblock Weight normalization: A simple reparameterization to accelerate
  training of deep neural networks.
\newblock \emph{Advances in neural information processing systems}, 29, 2016.

\bibitem[Shaw et~al.(2018)Shaw, Uszkoreit, and Vaswani]{shaw2018self}
Peter Shaw, Jakob Uszkoreit, and Ashish Vaswani.
\newblock Self-attention with relative position representations.
\newblock In \emph{Proceedings of the 2018 Conference of the North American
  Chapter of the Association for Computational Linguistics: Human Language
  Technologies, Volume 2 (Short Papers)}, pp.\  464--468, 2018.

\bibitem[Shleifer et~al.(2021)Shleifer, Weston, and
  Ott]{shleifer2021normformer}
Sam Shleifer, Jason Weston, and Myle Ott.
\newblock Normformer: Improved transformer pretraining with extra
  normalization.
\newblock \emph{arXiv preprint arXiv:2110.09456}, 2021.

\bibitem[Touvron et~al.(2021)Touvron, Cord, Douze, Massa, Sablayrolles, and
  J{\'e}gou]{touvron2021training}
Hugo Touvron, Matthieu Cord, Matthijs Douze, Francisco Massa, Alexandre
  Sablayrolles, and Herv{\'e} J{\'e}gou.
\newblock Training data-efficient image transformers \& distillation through
  attention.
\newblock In \emph{International Conference on Machine Learning}, pp.\
  10347--10357. PMLR, 2021.

\bibitem[Touvron et~al.(2022)Touvron, Cord, and
  J{\'{e}}gou]{DBLP:conf/eccv/TouvronCJ22}
Hugo Touvron, Matthieu Cord, and Herv{\'{e}} J{\'{e}}gou.
\newblock Deit {III:} revenge of the vit.
\newblock In Shai Avidan, Gabriel~J. Brostow, Moustapha Ciss{\'{e}},
  Giovanni~Maria Farinella, and Tal Hassner (eds.), \emph{Computer Vision -
  {ECCV} 2022: 17th European Conference, Tel Aviv, Israel, October 23-27, 2022,
  Proceedings, Part {XXIV}}, volume 13684 of \emph{Lecture Notes in Computer
  Science}, pp.\  516--533. Springer, 2022.
\newblock \doi{10.1007/978-3-031-20053-3\_30}.
\newblock URL \url{https://doi.org/10.1007/978-3-031-20053-3\_30}.

\bibitem[Vaswani et~al.(2017)Vaswani, Shazeer, Parmar, Uszkoreit, Jones, Gomez,
  Kaiser, and Polosukhin]{vaswani2017attention}
Ashish Vaswani, Noam Shazeer, Niki Parmar, Jakob Uszkoreit, Llion Jones,
  Aidan~N Gomez, {\L}ukasz Kaiser, and Illia Polosukhin.
\newblock Attention is all you need.
\newblock In \emph{Advances in neural information processing systems}, pp.\
  5998--6008, 2017.

\bibitem[Wang et~al.(2022)Wang, Ma, Dong, Huang, Zhang, and
  Wei]{wang2022deepnet}
Hongyu Wang, Shuming Ma, Li~Dong, Shaohan Huang, Dongdong Zhang, and Furu Wei.
\newblock Deepnet: Scaling transformers to 1,000 layers.
\newblock \emph{arXiv preprint arXiv:2203.00555}, 2022.

\bibitem[Yang et~al.(2022)Yang, Hu, Babuschkin, Sidor, Liu, Farhi, Ryder,
  Pachocki, Chen, and Gao]{yang2022tensor}
Greg Yang, Edward~J Hu, Igor Babuschkin, Szymon Sidor, Xiaodong Liu, David
  Farhi, Nick Ryder, Jakub Pachocki, Weizhu Chen, and Jianfeng Gao.
\newblock Tensor programs v: Tuning large neural networks via zero-shot
  hyperparameter transfer.
\newblock \emph{arXiv preprint arXiv:2203.03466}, 2022.

\bibitem[Yao et~al.(2020)Yao, Gholami, Keutzer, and
  Mahoney]{DBLP:conf/bigdataconf/YaoGKM20}
Zhewei Yao, Amir Gholami, Kurt Keutzer, and Michael~W. Mahoney.
\newblock Pyhessian: Neural networks through the lens of the hessian.
\newblock In Xintao Wu, Chris Jermaine, Li~Xiong, Xiaohua Hu, Olivera Kotevska,
  Siyuan Lu, Weija Xu, Srinivas Aluru, Chengxiang Zhai, Eyhab Al{-}Masri,
  Zhiyuan Chen, and Jeff Saltz (eds.), \emph{2020 {IEEE} International
  Conference on Big Data {(IEEE} BigData 2020), Atlanta, GA, USA, December
  10-13, 2020}, pp.\  581--590. {IEEE}, 2020.
\newblock \doi{10.1109/BigData50022.2020.9378171}.
\newblock URL \url{https://doi.org/10.1109/BigData50022.2020.9378171}.

\bibitem[You et~al.(2017)You, Gitman, and Ginsburg]{you2017large}
Yang You, Igor Gitman, and Boris Ginsburg.
\newblock Large batch training of convolutional networks.
\newblock \emph{arXiv preprint arXiv:1708.03888}, 2017.

\bibitem[Zhai et~al.(2022)Zhai, Kolesnikov, Houlsby, and
  Beyer]{DBLP:conf/cvpr/Zhai0HB22}
Xiaohua Zhai, Alexander Kolesnikov, Neil Houlsby, and Lucas Beyer.
\newblock Scaling vision transformers.
\newblock In \emph{{IEEE/CVF} Conference on Computer Vision and Pattern
  Recognition, {CVPR} 2022, New Orleans, LA, USA, June 18-24, 2022}, pp.\
  1204--1213. {IEEE}, 2022.
\newblock \doi{10.1109/CVPR52688.2022.01179}.
\newblock URL \url{https://doi.org/10.1109/CVPR52688.2022.01179}.

\bibitem[Zhang et~al.(2018)Zhang, Dauphin, and Ma]{zhang2018fixup}
Hongyi Zhang, Yann~N Dauphin, and Tengyu Ma.
\newblock Fixup initialization: Residual learning without normalization.
\newblock In \emph{International Conference on Learning Representations}, 2018.

\end{thebibliography}
\end{document}